\documentclass[default]{sn-jnl-preprint}
\jyear{2022}

\usepackage[utf8]{inputenc}
\usepackage{emile}
\usepackage{bm}

\theoremstyle{theorem}

\theoremstyle{theorem}

\newcommand{\op}{f}
\newcommand{\boostf}{\rho}
\newcommand{\truth}{{\boldsymbol{t}}}
\newcommand{\truths}{t}
\newcommand{\truthalt}{{\boldsymbol{u}}}
\newcommand{\truthsalt}{u}
\newcommand{\boostv}{{\hat{\truth}}}
\newcommand{\boosts}{{\hat{\truths}}}

\newcommand{\minboost}{\rho^*}
\newcommand{\minboostv}{{\truth^*}}
\newcommand{\minboosts}{{\truths^*}}
\newcommand{\weight}{t}
\newcommand{\weightmax}{\weight^{\max}}
\newcommand{\weightmin}{\weight^{\min}}

\newcommand{\revis}[1]{{\hat{t}_{#1}}}
\newcommand{\revismin}[1]{{\min_{#1}}}
\newcommand{\revismax}[1]{{\max_{#1}}}

\newcommand{\add}{g}

\newcommand{\newmax}{\boosts_{\max}}
\newcommand{\newmin}{\boosts_{\min}}

\newcommand{\boost}{refinement}
\newcommand{\Boost}{Refinement}
\newcommand{\boosted}{refined}
\newcommand{\Boosted}{Refined}

\newcommand{\lukincrease}{\lambda}

\newcommand{\predicates}{{\mathcal{P}}}
\newcommand{\constants}{{\mathcal{C}}}
\newcommand{\truthc}{{\boldsymbol{c}}}

\newcommand\todo[1]{}
\newcommand\tododeclarations[1]{\textcolor{red}{}}
\newcommand\ale[1]{{#1}}%\textcolor{blue}{#1}}
\newcommand\emile[1]{}%{\textcolor{green}{EvK: #1}}
\newcommand\emi[1]{{#1}}%\textcolor{gray}{#1}}
\newcommand\frank[1]{}%{\textcolor{orange}{FvH: #1}}
\newcommand\ls[1]{} %\textcolor{green!40!black}{LS: #1}}
\definecolor{codegreen}{rgb}{0,0.6,0}
\definecolor{codegray}{rgb}{0.5,0.5,0.5}
\definecolor{codepurple}{rgb}{0.58,0,0.82}
\definecolor{backcolour}{rgb}{0.95,0.95,0.92}

\lstdefinestyle{mystyle}{
    backgroundcolor=\color{backcolour},   
    commentstyle=\color{codegreen},
    keywordstyle=\color{magenta},
    numberstyle=\tiny\color{codegray},
    stringstyle=\color{codepurple},
    basicstyle=\ttfamily\footnotesize,
    breakatwhitespace=false,         
    breaklines=true,                 
    captionpos=b,                    
    keepspaces=true,                 
    numbers=left,                    
    numbersep=5pt,                  
    showspaces=false,                
    showstringspaces=false,
    showtabs=false,                  
    tabsize=2
}

\begin{document}

% \title[Article Title]{\emph{Logical Residual Layers}: Theory and Applications}
\title[Refining neural network predictions]{Refining neural network predictions using background knowledge}

\author[1]{\fnm{Alessandro} \sur{Daniele}} \email{daniele@fbk.eu}
\equalcont{These authors contributed equally to this work.}
\author[2]{\fnm{Emile} \sur{van Krieken}}\email{e.van.krieken@vu.nl}
\equalcont{These authors contributed equally to this work.}
\author[1]{\fnm{Luciano} \sur{Serafini}}
\author[2]{\fnm{Frank} \sur{van Harmelen}}

\affil[1]{\orgdiv{Data and Knowledge Management unit}, \orgname{Fondazione Bruno Kessler}, \orgaddress{\street{via Sommarive 18}, \city{Trento}, \postcode{38123}, \country{Italy}}}

\affil[2]{\orgdiv{Department of Computer Science}, \orgname{Vrije Universiteit Amsterdam}, \orgaddress{\street{de Boelelaan 1081a}, \city{Amsterdam}, \postcode{1081HV}, \country{Netherlands}}}

\abstract{Recent work has shown logical background knowledge can be used in learning systems to compensate for a lack of labeled training data. 
Many methods work by creating a loss function that encodes this knowledge. However, often the logic is discarded after training, even if it is still useful at test time. Instead, we ensure neural network predictions satisfy the knowledge by refining the predictions with an extra computation step. We introduce differentiable \emph{\boost\ functions} that find a corrected prediction close to the original prediction. 
We study how to effectively and efficiently compute these \boost\ functions. Using a new algorithm called Iterative Local Refinement (ILR), we combine \boost\ functions to find \boosted\ predictions for logical formulas of any complexity. ILR finds \boost s on complex SAT formulas in significantly fewer iterations and frequently finds solutions where gradient descent can not. Finally, ILR produces competitive results in the MNIST addition task.}

\maketitle

\section{Introduction}
% \frank{Neuro-symbolic systems promise to increase sample efficiency, transfer learning and explainability by using background knowledge as symbolically encoded bias. For this purpose, systems such as LTN, KENN, SBR-CC, DPL etc. try to enforce that predictions from the neural network satisfy the background knowledge. 
% The central question we study is how to find effective and efficient methods to ensure that predictions from the neural network satisfy the background knowledge.
% }
Recent years have shown promising examples of using symbolic background knowledge in learning systems: From training classifiers with weak supervision signals \cite{manhaeveDeepProbLogNeuralProbabilistic2018}, generalizing learned classifiers to new tasks \cite{roychowdhuryRegularizingDeepNetworks2021}, compensating for a lack of good supervised data \cite{diligentiSemanticbasedRegularizationLearning2017,donadelloLogicTensorNetworks2017}, to enforcing the structure of outputs through a logical specification \cite{xuSemanticLossFunction2018}. The main idea underlying these integrations of learning and reasoning, often called neuro-symbolic integration, is that background knowledge can complement the neural network when one lacks high-quality labeled data \cite{giunchigliaDeepLearningLogical2022}. Although pure deep learning approaches excel when learning over \emph{vast} quantities of data with \emph{gigantic} amounts of compute \cite{chowdheryPaLMScalingLanguage2022,rameshHierarchicalTextConditionalImage2022}, most tasks are not afforded this luxury. 

% Additionally, transparent reasoning systems can help create explainable and safe decisions, which is necessary when deploying models in the wild\todo{Do we have a source for this? It also feels rather random this statement.}.
Many neuro-symbolic methods work by creating a differentiable loss function that encodes the background knowledge (Figure \ref{fig:LRL-comparison}a). However, often the logic is discarded after training, even though this background knowledge could still be helpful at test time \cite{roychowdhuryRegularizingDeepNetworks2021,giunchigliaROADRAutonomousDriving2022}. Instead, we ensure we constrain the neural network with the background knowledge, both during train time and test time, by correcting its output such that it will satisfy the background knowledge (Figure \ref{fig:LRL-comparison}b). In particular, we consider how to make such corrections while being as close as possible to the original predictions of the neural network.

We study how to effectively and efficiently correct the neural network by ensuring its predictions satisfy the symbolic background knowledge. In particular, we consider fuzzy logics formed using functions called t-norms \cite{klementTriangularNorms2013,rossFuzzyLogicEngineering2010}. Prior work has shown how to use a gradient ascent-based optimization procedure to find a prediction that satisfies this fuzzy background knowledge \cite{diligentiSemanticbasedRegularizationLearning2017,roychowdhuryRegularizingDeepNetworks2021}. However, a recent model called KENN \cite{danieleKnowledgeEnhancedNeural2019} shows how to compute the correction analytically for a fragment of the G\"{o}del logic. 

To extend this line of work, we introduce the concept of \emph{\boost\ functions}, and derive \boost\ functions for many fuzzy logic operators. \Boost\ functions are functions that find a prediction that satisfies the background knowledge while staying close to the neural network's original prediction. Using a new algorithm called \emph{Iterative Local Refinement} (ILR), we can combine \boost\ functions for different fuzzy logic operators to efficiently find \boost s for logical formulas of any complexity. Since \boost\ functions are differentiable, we can easily integrate them as a neural network layer. In our experiments, we compare ILR with an approach using gradient ascent. We find that ILR finds optimal \boost s in significantly fewer iterations. \ale{Moreover, ILR often produces results that stay closer to the original predictions or better satisfy the background knowledge.
} Finally, we evaluated ILR on the MNIST Addition task~\cite{manhaeveDeepProbLogNeuralProbabilistic2018} and show that ILR can be combined with neural networks to solve common neuro-symbolic tasks.

In summary, our contributions are:
\begin{enumerate}
\item We formalize the concept of minimal \boost\ functions in Section \ref{sec:minimal-boost-function}.
\item We introduce the ILR algorithm in Section \ref{sec:ILR}, which uses the minimal \boost\ functions for individual fuzzy operators to find \boost s for general logical formulas.
\item We discuss how to use ILR for neuro-symbolic AI in Section \ref{sec:neuro-symbolic}, where we exploit the fact that ILR is a differentiable algorithm.
\item We analytically derive minimal \boost\ functions for individual fuzzy operators constructed from the G\"{o}del, \luk, and product t-norms in Section \ref{sec:basic-t-norm}. 
\item We discuss a large class of t-norms for which we can analytically derive minimal \boost\ functions in Section \ref{sec:general-analysis}.
\item We compare ILR to gradient descent approaches and show it finds \boost s on complex SAT formulas in significantly fewer iterations and frequently finds solutions where gradient descent can not. 
\item We apply ILR to the MNIST Addition task~\cite{manhaeveDeepProbLogNeuralProbabilistic2018} to test how ILR behaves when injecting knowledge into neural network models.
\end{enumerate}
% and gets stuck in local maxima less frequently \todo{summarize results properly}.

\label{sec:lrl}
\begin{figure}
    \includegraphics[width=\linewidth]{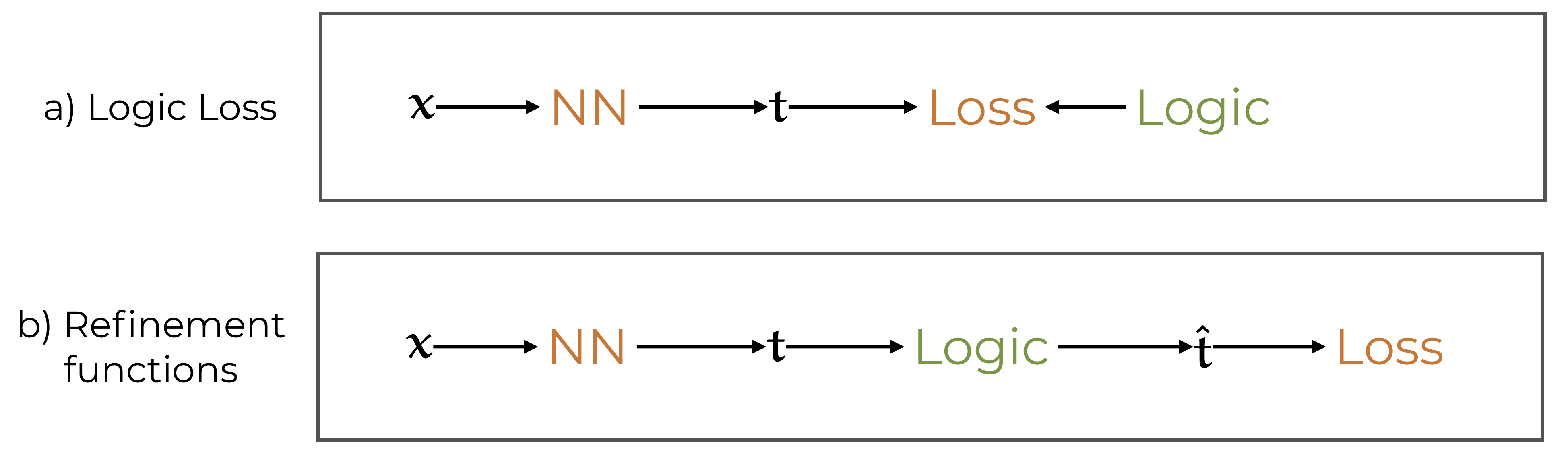}
    \caption{Comparing different approaches for constraining neural networks with background knowledge. Loss-based approaches include LTN, SBR, and Semantic Loss, while KENN, C-HMCNN(h), and SBR-CC are representatives for \boost\ functions.}
    \label{fig:LRL-comparison}
    %     \begin{luciano}
    %   \begin{itemize}
    %   \item the $t_{in}$ and $t_{out}$ in the version 2 and 3 do not look the same to me. and for the LRL they come out of the
    %     blue on the explanation. I don't understand why in LRL they cannot be considered as part of the $t_{in}$. Also the label ``out'' is very strange for a data with is provided only in input.
    %   \item How I understand the LRL is very symilar to DeepProblog. The only difference is that in DeepProblog you
    %     plug in an SDD that proves the truth value of $t_{out}$ based on the truth value of $t_{in}$ and
    %     as far as I know $t_{in}$ and $t_{out}$ are truth values about a disjoint set of atims, while
    %     LRL do the same, but instead of a proof (SDD) it add a boost function, and the input and output
    %     atoms are the same. 
    %   \end{itemize}
    % \end{luciano}
\end{figure}

\section{Related work}
ILR falls into a larger body of work that attempts to integrate background knowledge expressed as logical formulas into neural networks. For an overview, see \cite{giunchigliaDeepLearningLogical2022}. As shown in Figure \ref{fig:LRL-comparison}, methods can be categorized by whether they only use background knowledge during training in the form of a loss function~\cite{badreddineLogicTensorNetworks2022,xuSemanticLossFunction2018,diligentiSemanticbasedRegularizationLearning2017,fischerDL2TrainingQuerying,yangInjectingLogicalConstraints,vankriekenAnalyzingDifferentiableFuzzy2022} or whether the background knowledge is part of the model and therefore enforces the knowledge also at test time~\cite{danieleKnowledgeEnhancedNeural2019,wangSATNetBridgingDeep2019,giunchigliaMultiLabelClassificationNeural2021,ahmedSemanticProbabilisticLayers2022,hoernleMultiplexNetFullySatisfied2022,dragoneNeuroSymbolicConstraintProgramming}. ILR is a method in the second category. We note that these approaches can be combined~\cite{giunchigliaROADRAutonomousDriving2022,roychowdhuryRegularizingDeepNetworks2021}.

First, we discuss approaches that construct loss functions from the logical formulas (Figure \ref{fig:LRL-comparison}a). These loss functions measure when the deep learning model violates the background knowledge, such that minimizing the loss function amounts to \say{correcting} such violations \cite{vankriekenAnalyzingDifferentiableFuzzy2022}. 
While these methods show significant empirical improvement, they do not guarantee that the neural network will satisfy the formulas outside the training data. 
% Furthermore, these methods cannot be used to derive the truth value of additional propositions. 
LTN and SBR \cite{badreddineLogicTensorNetworks2022,diligentiSemanticbasedRegularizationLearning2017} use fuzzy logic to provide compatibility with neural network learning, while Semantic Loss \cite{xuSemanticLossFunction2018} uses probabilistic logics. The formalization of \boost\ functions can be extended to probabilistic logics by defining a suitable notion of minimality, like the KL-divergence between the original and \boosted\ distributions over ground atoms. 

Among the methods where knowledge is part of the model, KENN inspired \ale{ILR \cite{danieleKnowledgeEnhancedNeural2019,KENN_rel}. KENN is a framework that injects knowledge into neural networks by iteratively refining its predictions. It uses 
a relaxed version of the G\"{o}del t-conorm obtained through a relaxation of the argmax function, which it applies }
% which can be seen as   making corrections with the G\"{o}del t-conorm. KENN relaxes the computation by 
% softening the max function of the G\"{o}del t-conorm, 
% and applying the \boost\ vector 
in logit space. 
Closely related to both ILR and KENN is C-HMCNN(h) \cite{giunchigliaMultiLabelClassificationNeural2021}, which we see as computing the minimal \boost\ function for stratified normal logic programs under G\"{o}del t-norm semantics. We discuss this connection in more detail in Section \ref{sec:godel-t-norm}.

The loss-function based method SBR also introduces a procedure for using the logical formulas at test time in the context of collective classification~\cite{diligentiSemanticbasedRegularizationLearning2017,roychowdhuryRegularizingDeepNetworks2021}. Unlike KENN \cite{danieleKnowledgeEnhancedNeural2019}, these approaches do not enforce the background knowledge during training but only use it as a test time procedure. In particular, \cite{roychowdhuryRegularizingDeepNetworks2021} shows that doing these corrections at test time improves upon just using the loss-function approach. Unlike our analytic approach to \boost\ functions, SBR finds new predictions using a gradient descent procedure very similar to the algorithm we discuss in Section \ref{sec:gradient-descent}. We show it is much slower to compute than ILR.

Another method closely related to ILR is the neural network layer SATNet \cite{wangSATNetBridgingDeep2019}, which has a setup closely related to ours. However, SATNet does not have a notion like minimality and uses a different underlying logic constructed from a semidefinite relaxation. DeepProbLog \cite{manhaeveDeepProbLogNeuralProbabilistic2018} also is a probabilistic logic, but unlike Semantic Loss is used to derive new statements through proofs and cannot directly be used to correct the neural network on predictions that do not satisfy the background knowledge. Instead, ILR can be used both for injecting constraints on the output of a neural network, as well as for proving new statements starting from the neural network predictions.
% as shown in Figure \ref{fig:LRL-comparison}c. While DeepProbLog is very flexible and well-defined, probabilistic inference is expensive and hard to scale \cite{manhaeveApproximateInferenceNeural2021}, unlike LRLs. 

Finally, some methods are limited to equality and inequality constraints rather than general symbolic background knowledge \cite{fischerDL2TrainingQuerying,hoernleMultiplexNetFullySatisfied2022}. DL2 \cite{fischerDL2TrainingQuerying} combines these constraints into a real-valued loss function, while MultiplexNet \cite{hoernleMultiplexNetFullySatisfied2022} adds the knowledge as part of the model. However, MultiplexNet requires expressing the logical formulas as a DNF formula, which is hard to scale.

\todo{Ale: please add comparison LRNN/NLM/NMLN?}

\section{Fuzzy Operators}

% \begin{luciano}
%   \begin{itemize}
%     \item 
%   in this section there is a mixture of t-norm and t-conorm and residuum as binary operators, and as $n$-ary operators. You should better separate the two. By providing the the basic definition of t-norm and t-conorm as binary operator, and then saying that
%   due to associativity they can be extended to $n$-ary (aggregation) operators, and that in the rest of the paper we will consider these generalised version of t-norm and t-conorms.
% \item concerning the notation. I think that you should use bold for vector and non-bold for the single elements of the vector.
%   I.e., $\bm t=\left<t_1,\dots,t_n\right>$. This notation would be more coherent with the other symbols that you
%   are using in the text. Sometimens indeed you index also the vectors like $\bm x_n$, and this is not clear if you
%   are referring tho the $n$-ary vector or to the last element of such a vector.
% \item in the definition of t-norm, defining a t-norm in terms of the other t-norm is not a good practice. You should define them independently, and if it is interesting, or necessary, you can state properties about their relationships
% \end{itemize}
% \end{luciano}

\looseness=-1
We will first provide the necessary background knowledge for defining and analyzing minimal \boost\ functions. In particular, we will consider fuzzy operators, which generalize the connectives of classical boolean logic. 
For formal treatments of the study of such operators, we refer the reader to \cite{klementTriangularNorms2013}, which discusses t-norms and t-conorms, to \cite{jayaramFuzzyImplications2008} for fuzzy implications, to \cite{calvoAggregationOperatorsProperties2002} for aggregation functions, and to \cite{vankriekenAnalyzingDifferentiableFuzzy2022} for an analysis of the derivatives of these operators.

% \begin{table}[]
%     \centering
%     \begin{tabular}{llll}
%     \hline
%     Name          & T-norm & T-conorm \\ \hline 
%     Minimum & ${T_G}(\truth) = \min(\truths_1, ..., \truths_n)$ & ${S_G}(\truth) = \max(\truths_1, ..., \truths_n)$ \\ 
%     Product & ${T_P}(\truth) = \prod_{i=1}^n \truths_i$ & ${S_P}(\truth) = 1 - \prod_{i=1}^n(1 - \truths_i)$ \\ 
%     \luk  & ${T_{L}}(\truth) = \max(\sum_{i=1}^n \truths_i - (n - 1), 0)$ & ${S_{L}}(\truth) = \min\left(\sum_{i=1}^n \truths_i, 1\right)$ \\
%     Drastic & $T_D(\truth)=\begin{cases} \truths_i, \text{if } i={\arg\min}_{i}\truths_i \text{ and } \truths_j=1, j\neq i\\ 0, \text{otherwise.}\end{cases}$ &  
%     \end{tabular}
%     \caption{Some common t-norms extended to any-arity aggregation operators.}
%     \label{tab:t-norms}
% \end{table}

\begin{table}[]
    \centering
    \begin{tabular}{ll}
    \hline
    Name          & T-norm\\ \hline 
    Minimum & ${T_G}(\truth) = \min_{i=1}^n \truths_i$ \\ 
    Product & ${T_P}(\truth) = \prod_{i=1}^n \truths_i$  \\ 
    \luk{}  & ${T_{L}}(\truth) = \max(\sum_{i=1}^n \truths_i - (n - 1), 0)$ 
    %   Drastic & $T_D(\truth)=\begin{cases}t_i & \text{if $\truths_j=1$ for all $j\neq i$} \\
    %                            0&\text{otherwise.}\end{cases}$ \\
    % Nilpotent minimum & $T_{N}(\truth)=\begin{cases} \min_{i=1}^n \truths_i, &\text{if } \min_{i\neq j}(\truths_i+\truths_j) > 1, \\ 0, &\text{otherwise.}\end{cases}$
    \end{tabular}
    \caption{Some common t-norms extended to any-arity aggregation operators.}
    \label{tab:t-norms}
\end{table}

% In particular, we consider any-arity extensions of t-norms and t-conorms. 
\begin{definition}
    \label{deff:tnorm}
    A function $T: [0,1]^2\rightarrow [0, 1]$ is a \textit{t-norm} (triangular norm) if it is commutative, associative, increasing in both arguments, and if for all $t\in [0,1]$, $T(1, t) = t$.

    Similarly, a function $S: [0, 1]^2 \rightarrow [0, 1]$ is a \textit{t-conorm} if the last condition instead is that for all $t\in[0, 1]$, $S(0, t)=t$. 
\end{definition}
\emph{Dual} t-conorms are formed from a t-norm $T$ using $S(t_1, t_2)=1-T(1-t_1, 1-t_2)$. We list any-arity extensions, constructed using $T(\truth)=T(\truths_1, T(\truth_{2:n}))$, $T(t_i)=t_i$ of five basic t-norms in Table \ref{tab:t-norms}. Here $\truth=[\truths_1, ..., \truths_n]^\top\in [0, 1]^n$ is a vector of fuzzy truth values, which we will often refer to as \emph{(truth) vectors}. These any-arity extensions are examples of fuzzy aggregation operators \cite{calvoAggregationOperatorsProperties2002}. 
\begin{definition}
  A t-norm $T$ is \emph{Archimedean} if for all $x, y\in (0, 1)$, there is an $n$ such that
  $T(\underbrace{x,\dots,x}_{n\times})<y$.

  A continuous t-norm $T$ is \emph{strict} if, in addition, for all $x\in (0, 1)$, $0 < T(x, x) < x$. 
\end{definition}

\begin{definition}
    \label{def:implication}
    A function $I: [0, 1]^2\rightarrow [0, 1]$ is a \textit{fuzzy implication} if for all $t_1, t_2\in [0, 1]$, $I(\cdot, t_2)$ is decreasing, $I(t_1, \cdot)$ is increasing and if $I(0, 0) = 1$,  $I(1, 1) = 1$ and $I(1, 0) = 0$.
\end{definition}

% \looseness=-1
Note that fuzzy implications do not have $n$-ary extensions as they are not associative. The so-called  \textit{S-implications} are formed from the t-conorm by generalizing the material implication using $I(a, c)=S(1-a, c)$. 
Furthermore, every t-norm induces a unique \textit{residuum} or \textit{R-implication} \cite{jayaramFuzzyImplications2008} $R_T(a, c)= \sup \{ z \vert T(z, a) \leq c \}$.

Logical formulas $\varphi$ can be evaluated using compositions of fuzzy operators. We assume $\varphi$ is a propositional logic formula, but note this evaluation procedure can be extended to grounded first-order logical formulas on finite domains. \ale{For instance, \cite{KENN_rel} introduced a technique for propositionalizing universally quantified formulas of predicate logic in the context of KENN. Moreover, this technique can be extended to existential quantification} by treating it as a disjunction.  We assume a set of propositions $\predicates= \{P_1, ..., P_n\}$ and constants $\constants = \{C_1, ..., C_m\}$, where each constant has a fixed value $C_i\in [0, 1]$.

\begin{definition}
\label{def:evaluation}
If $T$ is a t-norm, $S$ a t-conorm and $I$ a fuzzy implication, then the \emph{fuzzy evaluation operator} $f_\varphi:[0, 1]^n\rightarrow [0,1]$ of the formula $\varphi$ with propositions $\predicates$ and constants $\constants$ is a function of truth vectors $\truth$ and given as
\begin{align}
    \op_{P_i}(\truth) &= \truths_i\\
    \op_{C_j}(\truth) &= C_j \\
    \op_{\neg \phi}(\truth) &= 1-f_{\phi}(\truth) \\
    \op_{\bigwedge_{j=1}^m \phi_j}(\truth) &= T(\op_{\phi_1}(\truth), ..., \op_{\phi_m}(\truth)) \\
    \op_{\bigvee_{j=1}^m \phi_j}(\truth) &= S(f_{\phi_1}(\truth), ..., f_{\phi_m}(\truth)) \\
    \op_{\phi\rightarrow \psi}(\truth) &= I(f_\phi(\truth), f_\psi(\truth)).
\end{align}
where we match on the structure of the formula $\varphi$ in the subscript $f_\varphi$. 
\end{definition}

\section{Minimal Fuzzy \Boost\ Functions}
\label{sec:minimal-boost-function}
We will next define (fuzzy) \boost\ functions, which consider how to change the input arguments of fuzzy operators such that the output of the operators is a given truth value. 
We prefer changes to the input arguments that are as small as possible. 
We will introduce several definitions to facilitate studying this concept. The first is an optimality criterion. 

\begin{definition}[Fuzzy \boost\ function]
    Let $\op_\varphi: [0, 1]^n\rightarrow [0,1]$ be a fuzzy evaluation operator. Then $\boostv: [0, 1]^n$ is called a \emph{\boosted\ (truth) vector} for the \emph{\boost\ value} $\revis{\varphi}\in[0, 1]$ if $\op_\varphi(\boostv) = \revis{\varphi}$.
    % \begin{enumerate}
    %     % \item $\truths_i\leq \truths_i + \boostf(\truth, \weight)_i \leq 1$, for all $i\in \{1, ..., n\}$, 
    %     \item $\op_\varphi(\boostv) = \revis{\varphi}$;
    %     \item for all $i\in F$, $\boosts_i=\truths_i$.
    %     % \item for all $i\in D$, $-\truths_i \leq \boosts_i\leq 1-\truths_i$.
    % \end{enumerate}

    Furthermore, let $\revismin{\varphi}=\min_{\boostv\in[0, 1]^n} \op_\varphi(\boostv)$ and $\revismax{\varphi}=\max_{\boostv \in [0, 1]^n} \op_\varphi(\boostv)$. 
    Then $\boostf: [0, 1]^{n}\times[0, 1]\rightarrow [0, 1]^n$ is a \emph{(fuzzy) \boost\ function}\footnote{\ale{The concept of \boost\ functions is closely related to the concept of \emph{Fuzzy boost function} in the KENN paper \cite{danieleKnowledgeEnhancedNeural2019}.}}
    % This is called a \emph{Fuzzy boost function} in the KENN paper \cite{danieleKnowledgeEnhancedNeural2019}.}
    for $\op_\varphi$ if for all $\truth \in [0, 1]^n$, 
    \begin{enumerate}
    \item for all $\revis{\varphi}\in [\revismin{\varphi}, \revismax{\varphi}]$, $\boostf(\truth, \revis{\varphi})$ is a \boosted\ vector for $\revis{\varphi}$;
    \item for all $\revis{\varphi} < \revismin{\varphi}$, $\boostf(\truth, \revis{\varphi})=\boostf(\truth, \revismin{\varphi})$;
    \item for all  $\revis{\varphi} > \revismax{\varphi}$, $\boostf(\truth, \revis{\varphi})=\boostf(\truth, \revismax{\varphi})$. 
    \end{enumerate}
    % \begin{enumerate}
    %     % \item $\truths_i\leq \truths_i + \boostf(\truth, \weight)_i \leq 1$, for all $i\in \{1, ..., n\}$, 
    %     \item $\op(\truth + \sign \boostf(\truth, \weight)) = \weight$;
    %     \item for all $i\in \{1, ..., n\}$, $-\truths_i \leq \delta(\truth, \weight)_i\leq 1-\truths_i$.
    % \end{enumerate}
    \end{definition}

% \ls{The second part of the definition does not work the problem is that $\weightmin$ and $\weightmax$
%   depends from $\truth$ therefore you cannot define the boost function from $[0,1]^n\times[\weightmin,\weightmax]$ I suggest the following modification}

% \begin{luciano}
%   \begin{definition}[Fuzzy boost function]
%     Let $\varphi$ be a propositional formula on a set of propositional variables $p_1,\dots,p_n$ and 
%     $\op_\varphi: [0, 1]^n\rightarrow [0,1]$ be a fuzzy evaluation operator;
%     $\boostv: [-1, 1]^n$ is called a \emph{fuzzy boost vector} for a truth value $\weight\in[0, 1]$ if 
% $\op_\varphi(\truth + \boostv) = \weight$ and $-t_i \leq \delta_i\leq 1-t_i$.

%   Let $\weightmin=\min_{\bm t}f_\varphi(\bm t)$ and 
%   $\weightmax=\max_{\bm t}f_\varphi(\bm t)$; a function 
%   $\boostf: [0, 1]^{n}\times[\weightmin, \weightmax]\rightarrow [0, 1]^n$ is a \emph{fuzzy boost function} for $\op_\varphi$
%   if for all $\truth \in [0, 1]^n$ and $\weight\in [\weightmin, \weightmax]$, $\boostf(\truth, \weight)$ is a fuzzy boost vector. 
% \end{definition}
% \end{luciano}

A \boost\ function for $f_\varphi$ changes the input truth vector in such a way that the new output of $f_\varphi$ will be $\revis{\varphi}$. 
Whenever $\revis{\varphi}$ is high, we want the \boosted\ vector to satisfy the formula $\varphi$, while if $\revis{\varphi}$ is low, we want it to satisfy its negation. When $\revis{\varphi}=1$, the constraint created by the formula is a hard constraint, while if it is in $(0, 1)$, this constraint is soft. 
% The second condition ensures the fixed truth values in $F$ remain the same. 
We require bounding the set of possible $\revis{\varphi}$ by $\revismin{\varphi}$ and $\revismax{\varphi}$, since if there are constants $C_i$, or if $\varphi$ has no satisfying (discrete) solutions, there can be formulas such that there can be no \boosted\ vectors $\boostv$ for which $f_\varphi(\boostv)$ equals 1. 

Next, we introduce a notion of minimality of \boost\ functions. The intuition behind this concept is that we prefer the new output, the \boosted\ vector $\boostv$, to stay as close as possible to the original truth vector $\truth$. Therefore, we assume we want to find a truth vector near the neural network's output that satisfies the background knowledge.
\begin{definition}[Minimal \boost\ function]
    Let $\minboost$ be a \boost\ function for operator $\op_\varphi$. $\minboost$ is a \emph{minimal} \boost\ function with respect to some norm $\|\cdot\|$ if for each $\truth\in[0, 1]^n$ and $\revis{\varphi}\in [\revismin{\varphi}, \revismax{\varphi}]$, there is no \boosted\ vector $\boostv'$ for $\revis{\varphi}$  such that $\|\minboost(\truth, \revis{\varphi}) - \truth\| > \|\boostv' - \truth\|$.
    % \begin{equation}
    %     ||\minboost(\truth, w)|| \leq ||\boostf'(\truth, w)||, \quad \forall \truth \in [0, 1]^n, \weight\in [\op(\truth), 1]
    % \end{equation}
\end{definition}

% In this paper, we will be discussing the $p$-norms (also known as $\ell_p$-norms) $\|\cdot\|_p= \left(\sum_{i=1}^n \| \boostf(\truth, \weight) \|^p \right)^{1/p}$%, for $p\geq 1$. Since by definition $\boostf(\truth, \weight)_i \geq 0$, we drop the absolute value in the remainder of the paper.

For a particular fuzzy evaluation operator $\op_\varphi$, finding the minimal \boost\ function corresponds to solving the following optimization problem:

\begin{equation}
    \label{eq:optim-problem}
    \begin{aligned}
        \textrm{For all } \quad & \truth\in [0,1]^n, \revis{\varphi} \in [\revismin{\varphi}, \revismax{\varphi}]  \\
        \min_{\boostv} \quad & \|\boostv - \truth\|  \\
        \textrm{such that } \quad & \op_\varphi(\boostv) = \revis{\varphi},  \\
        & 0 \leq  \boosts_i \leq 1
        % & \boosts_i=\truths_i, i\in F 
    \end{aligned}
\end{equation}

For some $f_\varphi$ we can solve this problem analytically using the Karush-Kuhn-Tucker (KKT) conditions. However, while $\|\cdot\|$ is convex, $f_\varphi$ (usually) is not. Therefore, we can not rely on efficient convex solvers. Furthermore, for strict t-norms, finding exact solutions to this problem is equivalent to solving PMaxSAT when $\revis{\varphi}=1$ \cite{diligentiSemanticbasedRegularizationLearning2017,giunchigliaROADRAutonomousDriving2022}, hence this problem is NP-complete. In Sections \ref{sec:general-analysis} and \ref{sec:class-analysis}, we will analytically derive minimal \boost\ functions for a large amount of individual fuzzy operators. These results are the theoretical contribution of this paper. We first discuss in Section \ref{sec:ILR} a method called ILR for finding general solutions to the problem of finding minimal \boost\ functions. ILR uses the analytical minimal \boost\ functions of individual fuzzy operators in a forward-backward algorithm. Then, in Section \ref{sec:neuro-symbolic}, we discuss how to use this algorithm for neuro-symbolic AI.

\section{Iterative Local Refinement}
\label{sec:ILR}
\begin{figure}
    \includegraphics[width=\linewidth]{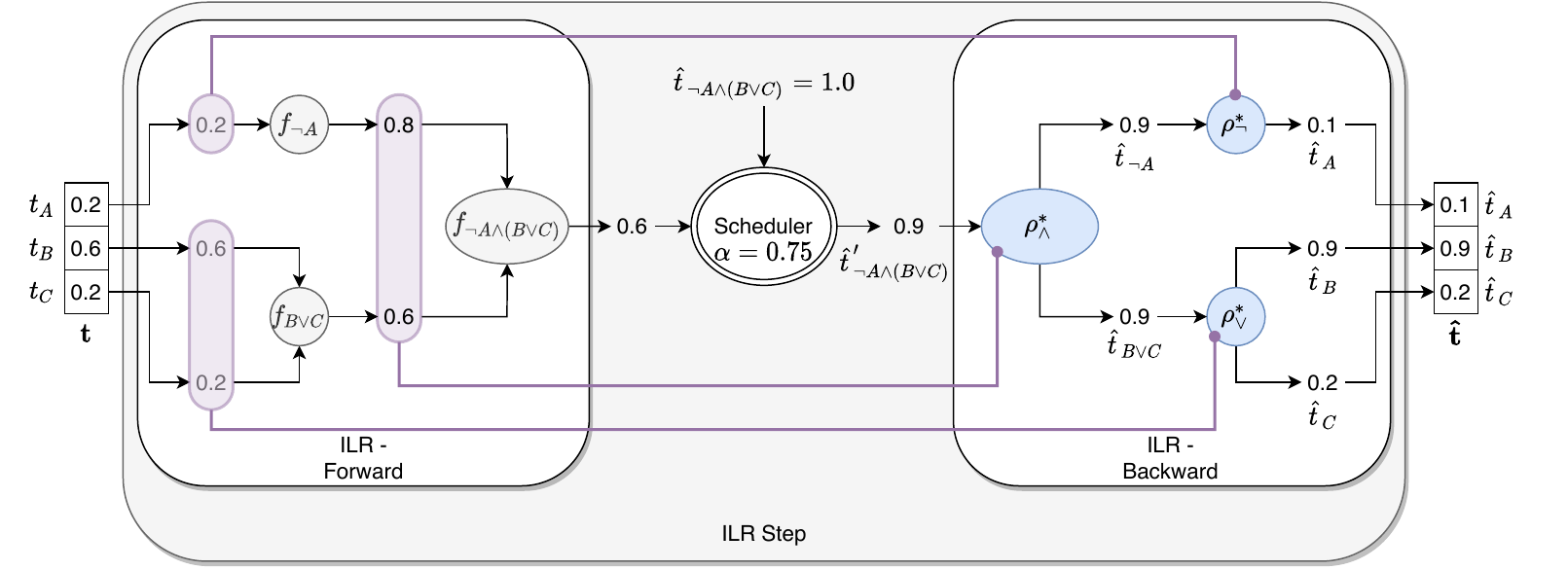}
    \caption{\ale{Visualization of one step of ILR for the G\"{o}del logic and formula $\phi = \lnot A \land (B \lor C)$. In the forward pass (left), ILR computes the truth value of $\phi$. In the backward pass (right), ILR traverses the computational graph of the forward step in reverse to calculate the \boosted\ vector $\boostv$. ILR substitutes each fuzzy operator of the forward pass with the corresponding \boost\ function. Each \boost\ function receives as input the initial truth values used by the fuzzy operator in the forward step (purple lines) and the target value for the corresponding subformula. The scheduler calculates the target value $\hat{t}'_{\lnot A \land (B \lor C)}$ for the entire formula, which ILR calls between the forward and backward steps.}}
    \label{fig:ILR}
\end{figure}

We introduce a fast, iterative, differentiable but approximate algorithm called \emph{Iterative Local Refinement (ILR)} that finds minimal \boost\ functions for general formulas. ILR is a forward-backward algorithm acting on the computation graph of formulas. First, it traverses the graph from its leaves to its root to compute the current truth values of subformulas. Then, it traverses the graph back from its root to the leaves to compute new truth values for the subformulas. ILR makes use of analytical minimal \boost\ functions to perform this backward pass.  ILR is a differentiable algorithm if the fuzzy operators and their corresponding minimal \boost\ functions are differentiable as it computes compositions of these functions. \ale{An example of one step of the ILR algorithm is presented in Figure~\ref{fig:ILR}, while Algorithm~\ref{alg:main} contains the entire pseudocode.}

% ILR computes \boosted\ vectors $\boostv$ for general propositional formulas $\varphi$. 
First, ILR computes the truth value of the formula in the forward pass \ale{(left side of Figure~\ref{fig:ILR})}, where the truth vectors of intermediate subformulas \ale{are saved (in Algorithm~\ref{alg:main}: $\truth_{\mathsf{sub}}$. In Figure~\ref{fig:ILR}: numbers inside colored boxes)}. Then, ILR computes a backward pass \ale{(right side of Figure~\ref{fig:ILR})}. ILR uses previously computed truth vectors of subformulas to compute the minimal \boosted\ vectors for the components of that subformula.  
We use the results from Sections \ref{sec:basic-t-norm} to compute these for the G\"{o}del, \luk\ and product fuzzy operators. 

In lines \ref{alg-line:backward-t-norm-start} to \ref{alg-line:backward-t-norm-end}, ILR computes the minimal \boosted\ vector for a conjunction of subformulas. We retrieve the truth values of the subformulas from the forward pass and call the minimal \boost\ function $\minboost_T$ for the chosen t-norm. This procedure gives us a \boosted\ vector, where each value corresponds to the \boosted\ value of a subformula. ILR then goes in recursion on the subformulas. Note that the pseudocode for disjunction and implication is analogous.  

One choice in ILR is how to combine the results from different subformulas. When a proposition appears in multiple subformulas, it can be assigned multiple different \boosted\ values. We found the heuristic in line \ref{alg-line:combine} generally works well, which takes the $\boosts_j$ with the largest absolute value. We also explored averaging the different \boosted\ values, but this took significantly longer to converge. Another choice is the convergence criterion. A simple option is to stop running the algorithm whenever it has stopped getting closer to the \boosted\ value for a couple of iterations. In our experiments, we observed that ILR monotonically decreases the distance to the \boosted\ value, after which it gets stuck either on a single local optimum or oscillates between two local minima.
% \begin{enumerate}
%     \item \emph{Max}: Take the $\boosts_i$ with the largest absolute value.
%     % \item \emph{Min}: Takes the $\boosts_i\neq 0$ with the smallest absolute value.
%     \item \emph{Mean}: Take the average over all $\boosts_i\neq \truths_i$. 
%     % \item \emph{Randomized}: Randomly chooses between Max and Min combination. 
% \end{enumerate}

Moreover, we experimented with a scheduling mechanism to smooth the behavior of ILR. We implement this in line \ref{alg-line:scheduling}. The scheduling mechanism works by choosing a different \boosted\ value at each iteration: The difference between the current truth value and the \boosted\ value is multiplied by a scheduling parameter $\alpha$, which we choose to be either 0.1 or 1 (no scheduling). While usually not necessary, for some formulas, the scheduling mechanism allowed for finding better solutions.

ILR is not guaranteed to find a \boosted\ vector $\boostv$ such that $\op_\varphi(\boostv)=\revis{\varphi}$. This is easy to see theoretically because, for many fuzzy logics like the product and G\"{o}del logics, $\revis{\varphi}=1$ corresponds to the PMaxSAT problem, which is NP-complete \cite{diligentiSemanticbasedRegularizationLearning2017,giunchigliaROADRAutonomousDriving2022}, while ILR has linear time complexity. 
% Furthermore  is also not guaranteed to be the minimal solution. 
However, this is traded off by 1) being highly efficient, usually requiring only a couple of iterations for convergence, and 2) not having any hyperparameters to tune, except arguably for the combination function. Furthermore, ILR usually converges quickly in neuro-symbolic settings since background knowledge is very structured, and the solution space is relatively dense. These settings are unlike the randomly generated SAT problems we study in Section \ref{sec:results-sat}. These contain little structure the ILR algorithm can exploit. 

% We would also like to highlight several properties and uses of this algorithm in Neuro-Symbolic settings. 

\begin{algorithm}
\caption{Iterative Local Refinement}\label{alg:main}
\begin{algorithmic}[1]
    \Require{$\varphi, \revis{\varphi}, \truth$, $\alpha\in (0, 1]$}
    \State $\truth' \gets \truth$
    \While{not converged}
        \State $\truth_{\mathsf{sub}}\gets \{\}$
        \For{subformula $\phi$ of $\varphi$}
            \State $\truth_{\mathsf{sub}}[\phi]\gets \op_\phi(\truth')$ \Comment{Forward pass using Definition \ref{def:evaluation}}
        \EndFor
        % \State $\revis{\varphi}' \gets(1 - \alpha) f_\varphi(\truth) + \alpha \revis{\varphi}$
        \State $\revis{\varphi}' \gets f_\varphi(\truth) + (\revis{\varphi} - f_\varphi(\truth))\alpha$ \label{alg-line:scheduling}
        % \State $\truth \gets \truth_{\mathsf{sub}}S[\varphi]$
        \State $\truth' \gets$ \Call{Backward}{$\varphi$, $\revis{\varphi}'$, $\truth_{\mathsf{sub}}$}
    \EndWhile
    \State \Return $\truth'$
    \Function{Backward}{$P_i$, $\revis{P_i}$, $\truth_{\mathsf{sub}}$}
        \State \Return $[\truths_1, \dots, \revis{P_i}, \dots, \truths_n]^\top$ \Comment{$\truth$ except at position $i$.}
    \EndFunction
    \Function{Backward}{$\neg \phi$, $\revis{\neg \phi}$, $\truth_{\mathsf{sub}}$}
        \State \Return \Call{Backward}{$\phi$, $1-\revis{\neg \phi}$, $\truth_{\mathsf{sub}}$}
    \EndFunction
    \Function{Backward}{$\bigwedge_{i=1}^m\phi_i$, $\revis{\varphi}$, $\truth_{\mathsf{sub}}$} \label{alg-line:backward-t-norm-start}
        % \State $\truth_\wedge \gets \left[\truth_{\mathsf{sub}}[\phi_1], ..., \truth_{\mathsf{sub}}[\phi_m]\right]^ \top$
        \State $\boostv_{\wedge}\gets \minboost_T(\left[\truth_{\mathsf{sub}}[\phi_1], ..., \truth_{\mathsf{sub}}[\phi_m]\right]^ \top, \revis{\varphi})$ \Comment{Minimal \boost\ function}
        \State $\boostv \gets \boldsymbol{0}$
        \For{$i\gets 1$ to $m$}
            \State $\boostv' \gets$ \Call{Backward}{$\phi_i$, $\boosts_{\wedge, i}$, $\truth_{\mathsf{sub}}$}
            \State $\boosts_j \gets$ \algorithmicif\ $\vert \boosts_j \vert > \vert \boosts'_j \vert$ \algorithmicthen\ $\boosts_j$ \algorithmicelse\ $\boosts' _j$ for all $j\in \{1, ..., n\}$ \label{alg-line:combine}
        \EndFor
        \State \Return $\boostv$ \label{alg-line:backward-t-norm-end}
    \EndFunction
\end{algorithmic}
\end{algorithm}

\section{Neuro-Symbolic AI using ILR}
\label{sec:neuro-symbolic}
The ILR algorithm can be added as a module after a neural network $g$ to create a neuro-symbolic AI model. The neural network predicts (possibly some of) the initial truth values $\truth$. Since both the forward and backward passes of ILR are differentiable computations, we can treat ILR as a constrained output layer \cite{giunchigliaDeepLearningLogical2022}. For instance, in Figure~\ref{fig:ILR}, the input $\truth$ could be generated by the neural network, and we provide supervision directly on the predictions $\boostv$. ILR ensures the predictions, i.e., the \boosted\ vector $\boostv$, satisfy the background knowledge while staying close to the original predictions made by the neural network. Loss functions like cross-entropy can use $\boostv$ as the prediction. We train the neural network $g$ by minimizing the loss function with gradient descent and backpropagating through the ILR layer. 

One strength of ILR is the flexibility of the \boost\ values $\revis{\varphi_i}$ for each formula $\varphi_i$. These can be set to 1 to treat $\varphi_i$ as a hard constraint that always needs to be satisfied. Alternatively, \boost\ values can be trained as part of a larger deep learning model. Since ILR is a differentiable layer, we can compute gradients of the \boost\ values. This procedure allows ILR to learn what formulas are useful for prediction. For instance, in Figure~\ref{fig:ILR}, $\revis{\lnot A \land (B \lor C)}$ can either be given or act as a parameter of the model that is learned together with the neural network parameters.

We give an example of the integration of ILR with a neural network in Figure~\ref{fig:ILR_MNIST}, where we use ILR for the MNIST Addition task proposed by~\cite{manhaeveDeepProbLogNeuralProbabilistic2018}. In this task, we have access to a training set composed of triplets $(x, y, z)$, where $x$ and $y$ are images of MNIST~\cite{lecunMNISTHandwrittenDigit2010} handwritten digits, and $z$ is a label representing an integer in the range $\{0,...,18\}$, corresponding to the sum of the digits represented by $x$ and $y$. The task consists of learning the addition function and a classifier for the MNIST digits, with supervision only on the sums. To achieve this, knowledge consisting of the rules of addition is given. For instance, the rule
$
Is(x, 3) \land Is(y, 2) \to Is(x+y, 5)
$ states that the sum of 3 and 2 is 5.

The architecture of Figure~\ref{fig:ILR_MNIST} consists of a neural network (a CNN) that performs digit recognition on the inputs $x$ and $y$. After this step, ILR predicts a truth value for each possible sum. Notice that we define the CNN outputs $\boldsymbol{C}_x,\boldsymbol{C}_y\in[0, 1]^{10}$ as constants, i.e., ILR does not change the predictions of the digits. Moreover, the initial prediction for the truth vector of possible sums $\truth_{x+y}\in[0, 1]^{19}$ is the zero vector. This allows ILR to act as a proof-based method. Indeed, similarly to DeepProbLog~\cite{manhaeveDeepProbLogNeuralProbabilistic2018}, the architecture proposed in Figure~\ref{fig:ILR_MNIST} uses the knowledge in combination with the predictions of the neural network to derive truth values for new statements (the sum of the two digits). We apply the loss function to the final predictions $\hat{\truth}_{x+y}$. During learning, the error is back-propagated through the entire model, reaching the CNN, which learns to classify the MNIST images from indirect supervision.

We present the results obtained by ILR in Section~\ref{sec:MNIST_exp}, and compare its performance with other neuro-symbolic AI frameworks.

\begin{figure}
    \includegraphics[width=\linewidth]{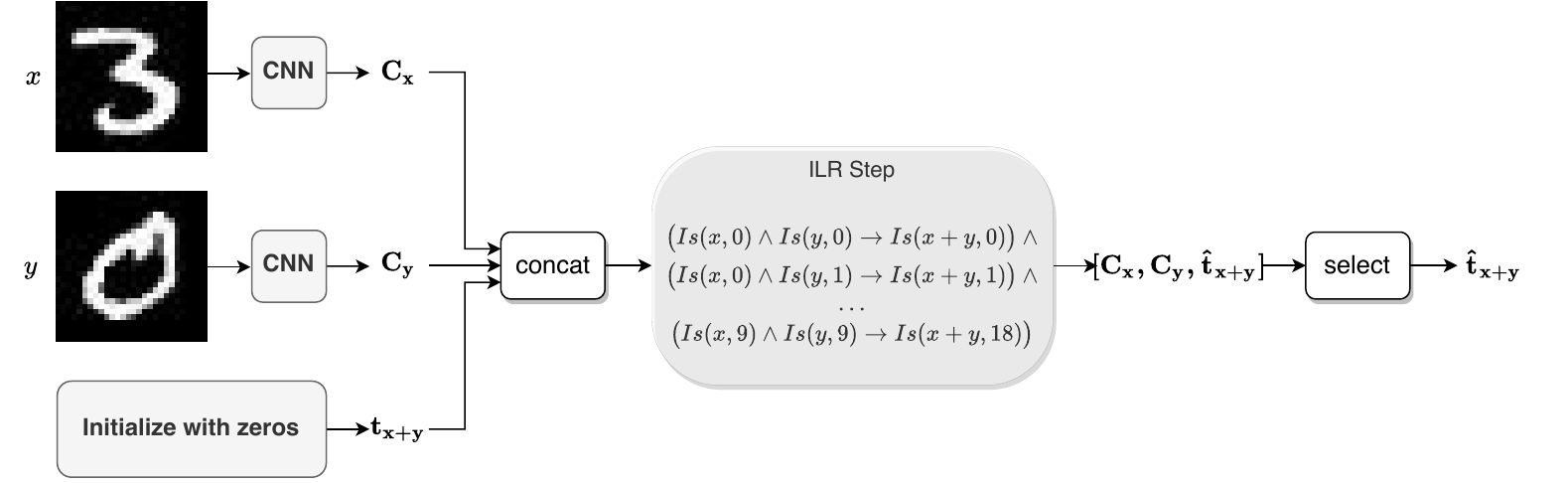}
    \caption{Neuro-symbolic architecture based on ILR for the MNIST Addition task. A CNN takes in input two images of MNIST digits, returning their classification. The predictions of the CNN are concatenated together with a vector of zeros, representing the initial prediction for the Addition task. We perform an ILR step to update the sum of the two numbers, which is the final output of the model.}
    \label{fig:ILR_MNIST}
\end{figure}

\section{Analytical minimal \boost\ functions}
\label{sec:general-analysis}
Having introduced the ILR algorithm, we next study the problem of finding minimal \boost\ functions for individual fuzzy operators. We need these in closed form to compute the ILR algorithm, as ILR uses them during the backward pass. This section first discusses several transformations of minimal \boost\ functions and gives the minimal \boost\ functions of the basic t-norms G\"{o}del, \luk\, and product. In Section \ref{sec:general-analysis} we investigate a large class of t-norms for which we have closed-form formulas for the minimal \boost\ functions.
\subsection{General results}
We first provide several basic results on minimal \boost\ functions for fuzzy operators. In particular, we will consider formulas such as $\varphi=\bigwedge_{i=1}^n P_i \bigwedge_{i=1}^m C_i$, that is, conjunctions of propositions and constants. 
% First, let $\truth_F$ be the truth vector $\truth$ restricted to the indices in $F$. 
As abuse of notation, from here on, we will refer to $\revismin{\varphi}$ and $\revismax{\varphi}$ when evaluated by the t-norm $T$ as $\revismin{T}$ and $\revismax{T}$ and will do so also for other fuzzy operators.
We find using Definition \ref{deff:tnorm} that for some t-norm $T$,  $\revismin{T} = 0$ and $\revismax{T} = T(\truthc)$, where $\truthc$ is the  values of the constants $C_1, ..., C_m$ as a truth vector, while for some t-conorm $S$, $\revismin{S} = S(\truthc)$ and $\revismax{S}=1$. Note that for $m=0$, $\revismax{T}=1$ and $\revismin{S}=0$. Next, we find some useful transformations of minimal \boost\ functions to derive new results:

\begin{proposition}
    \label{prop:dual-t-conorm}
    Consider the formulas $\phi=\bigwedge_{i=1}^n P_i\bigwedge_{i=1}^m C_i$ and $\psi= \neg (\bigvee_{i=1}^n P_i \bigvee_{i=1}^m C_i)$. Assume $\minboost_\phi$ is a minimal \boost\ function for $\op_\phi$ evaluated using t-norm $T$. Consider $\op_{\psi}(\truth)$ evaluated using dual t-conorm $S$ of $T$. Then $\minboost_{\psi}(\truth, \revis{\psi})=\boldsymbol{1}-\minboost_{\phi}(\boldsymbol{1}-\truth, \revis{\psi})$ is a minimal \boost\ function for $\op_{\psi}$.
\end{proposition}
\begin{proof}
    % Since $\op(\truth)= T(1-\truth)$, by assumption the minimal boost function is $\delta^*(1-\truth, w)$ such that $T(1-\truth +\minboost_T(1-\truth, \weight)) = \weight$. 
    First, note $\op_{\psi}(\truth) = 1-S(\truth, \truthc) = 1-(1-T(\boldsymbol{1}-\truth, \boldsymbol{1}-
\truthc))=T(\boldsymbol{1}-\truth, \boldsymbol{1}-\truthc)$. Consider $\truth'=\boldsymbol{1}-\truth$. By the assumption of the proposition, $\minboost_\phi(\truth', \revis{\phi})$ is a minimal \boost\ function for $T(\truth', \boldsymbol{1}-\truthc)=T(\boldsymbol{1}-\truth, \boldsymbol{1}-\truthc)=\op_\psi(\truth)$. Furthermore, note that 
    \begin{align*}
        \op_\psi(\minboost_{\psi}(\truth, \revis{\psi}))=T(\boldsymbol{1}-\minboost_{\psi}(\truth, \revis{\psi}), \boldsymbol{1}-\truthc) =T(\minboost_\phi(\truth', \revis{\psi}), \boldsymbol{1}-\truthc) = \revis{\psi}
    \end{align*}
\end{proof}
An analogous argument can be made for $\phi'=\bigvee_{i=1}^n P_i\bigvee_{i=1}^m C_i$ and $\psi=\neg(\bigwedge_{i=1}^n P_i\bigwedge_{i=1}^m C_i)$ to show that, given minimal \boost\ function $\minboost_{\phi'}$ of dual t-conorm $S$, the minimal \boost\ function for $\op_{\psi}(\truth)$ is $\minboost_{\psi}(\truth, \revis{\psi})=\boldsymbol{1}-\minboost_{\phi}(\boldsymbol{1}-\truth, \revis{\psi})$.
% Note that $\minboost_{\psi}$ and $\minboost_{\psi'}$ are non-positive functions since $1-S(\truth)$ is decreasing in all arguments.

We will use this result to simplify the process of finding minimal \boost\ functions for the t-norms and dual t-conorms. For example, assume we have a minimal \boost\ function $\minboost_T$ for $\revis{T}\in [T(\truth), \revismax{T}]$. Let $S$ be the corresponding dual t-conorm. Then, we can change the constraint $S(\boostv, \truthc)=\revis{S}$ in Equation \ref{eq:optim-problem} to the equivalent constraint $\boldsymbol{1}-S(\boostv, \truthc)=\boldsymbol{1}-\revis{S}$. We then use Proposition \ref{prop:dual-t-conorm} to find the minimal \boosted\ vector for $\revis{S}\in[\revismin{S}, S(\truth)]$ as $\boldsymbol{1}-\minboost_T(\boldsymbol{1}-\truth, 1-\revis{S})$. \todo{What if $1-\weight \not \in [\weightmin, \weightmax]$? Then this is technically not defined...}

\begin{proposition}
    \label{prop:s-implication}
    Consider the formulas $\phi= P_1 \vee P_2$ and $\psi= \neg P_1 \vee P_2$. Assume $\minboost_\phi$ is a minimal \boost\ function for $\op_\phi$ evaluated using the t-conorm $S$, and define $\truth'=[1-\truths_1, \truths_2]$. Then $\minboost_{\psi}(\truth, \revis{\psi})=\left[1-\minboost_{\phi}(\truth', \revis{\psi})_1, \minboost_{\phi}(\truth', \revis{\psi})_2\right]^\top$ is a minimal \boost\ function for $\op_{\psi}$.
\end{proposition}
\begin{proof}
    % Since $\op(\truth)= T(1-\truth)$, by assumption the minimal boost function is $\delta^*(1-\truth, w)$ such that $T(1-\truth +\minboost_T(1-\truth, \weight)) = \weight$. 
    First, note $\op_{\psi}(\truth) = S(1-\truth_1, \truth_2)$. By the assumption of the proposition, $\minboost_\phi(\truth', \revis{\psi})$ is a minimal \boost\ function for $S(\truth')=\op_\psi(\truth)$. Furthermore, note that 
    \begin{align*}
        \op_\psi(\minboost_{\psi}(\truth, \revis{\psi})) &= S(1-\minboost_\psi(\truth', \revis{\psi})_1, \minboost_\psi(\truth', \revis{\psi})_2)   \\
        &= S(1-(1-\minboost_\phi(\truth', \revis{\psi})_1), \minboost_\phi(\truth', \revis{\psi})_2)  = S(\minboost_\phi(\truth', \revis{\psi})) = \revis{\psi}.
    \end{align*}
\end{proof}
Similar to the previous proposition, this proposition gives us a simple procedure for finding the minimal \boost\ functions for the S-implication of some t-conorm.

\subsection{Basic T-norms}
\label{sec:basic-t-norm}
In this section, we introduce the minimal \boost\ functions for the t-norms and t-conorms of the three main fuzzy logics (G\"{o}del, \luk, and Product). In particular, we consider when these t-norms and t-conorms can act on both propositions and constants, that is, $\varphi=\bigwedge_{i=1}^n \truths_i \bigwedge_{i=1}^m C_i$, which is evaluated with $T(\truth, \truthc)$. We present the main results with simple examples. 
% For formal proofs of the highlighted results we refer the reader to Appendix\todo{add appendix and move formal proofs based on KKT there.}

\subsubsection{G\"{o}del t-norm}
\label{sec:godel-t-norm}
In this section, we derive minimal \boost\ functions for the G\"{o}del t-norm and t-conorm for the family of $p$-norms.

\begin{proposition}
    The minimal \boost\ function of the G\"{o}del t-norm for $\revis{T_G}\in [0, \min_{i=1}^m C_i]$ is
\begin{equation}
    \minboost_{T_G}(\truth, \revis{T_G})_i=\begin{cases}
        \revis{T_G}  & \text{if } \revis{T_G} \geq T_G(\truth) \text { and } \truths_i < \revis{T_G}, \\
        \revis{T_G}  & \text{if } \revis{T_G} < T_G(\truth) \text { and } i=\arg\min_{j=1}^n \truths_j, \\
        \truths_i & \text {otherwise,}
    \end{cases}
\end{equation} 
The minimal \boost\ function of the G\"{o}del t-conorm and $\revis{S_G} \in [\max_{i=1}^m C_i, 1]$ is 
\begin{equation}
    \minboost_{S_G}(\truth, \revis{S_G})_i=\begin{cases}
        \revis{S_G} & \text{if } \revis{S_G} \geq S_G(\truth) \text { and } i=\arg\max_{j=1}^m \truths_j, \\
        \revis{S_G} & \text{if } \revis{S_G} < S_G(\truth) \text { and } \truths_i > \revis{S_G}, \\
        0 & \text {otherwise.}
    \end{cases}
\end{equation} 
\end{proposition}
\begin{proof}
    Follows from Propositions \ref{prop:dual-t-conorm}, \ref{prop:godel-t-norm} and \ref{prop:godel-t-conorm}, see Appendix \ref{appendix:godel-t-norm} and \ref{appendix:godel-t-conorm}.
\end{proof}

\begin{proposition}
    \label{prop:godel-impl}
    A minimal \boost\ function of the G\"{o}del implication $R_G(\truths_1, \truths_2)=\begin{cases}\truths_2 & \text{if } \truths_1 > \truths_2, \\ 1 & \text{otherwise.}\end{cases}$ for $\revis{R_G}\in[\revismin{R_G}, \revismax{R_G}]$ is 
    \begin{equation}
            \minboost_{R_G}(\truths_1, \truths_2, \revis{R_G}) = \begin{cases}
                [\max(\revis{R_G} + \epsilon, \truths_1), \revis{R_G}]^\top  & \text{if } \revis{R_G} < 1 \\
                % [\max(\revis{R_G} + \epsilon, \truths_1), \revis{R_G}]^\top  & \text{if } \revis{R_G} < 1 \\
                [\truths_1, \max(\truths_1, \truths_2)]^\top & \text{otherwise.} 
            \end{cases}
    \end{equation}
    where $\epsilon$ is an arbitrarily small positive number.
\end{proposition}
The proof is in Appendix \ref{prop:godel-impl}.

\begin{figure}
    \includegraphics[width=\linewidth]{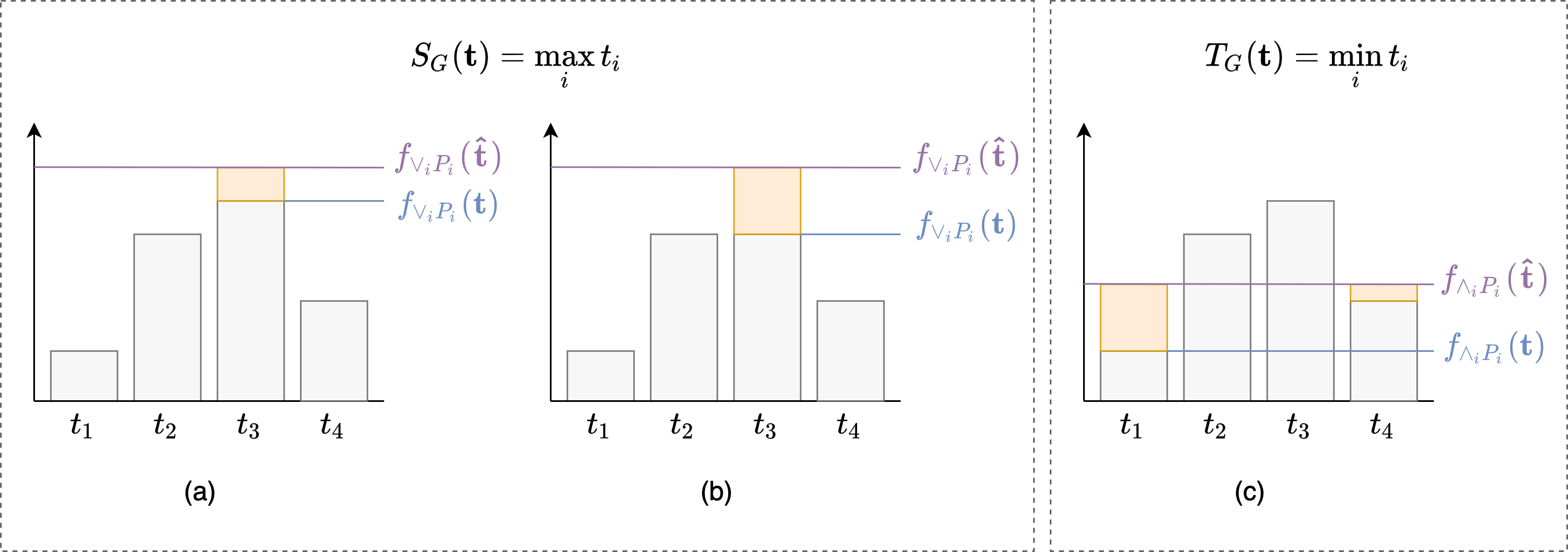}
    \caption{G\"{o}del minimal \boost\ functions. The grey bars represent the initial truth vectors $\truth$; the light blue and purple lines indicate the initial truth value of the formula and the revision value $\revis{\varphi}$, and the orange bars are the corresponding minimal \boosted\ vectors. (a) t-conorm; (b) t-conorm with two literals with same truth value; (c) t-norm.}
    \label{fig:Godel}
\end{figure}

The bar plot in Figure~\ref{fig:Godel}(a) shows an example for the G\"{o}del t-conorm with four literals. The minimal \boosted\ vector is represented with the orange boxes, while the initial and \boost\ values of the entire formula are represented as a blue and purple line respectively. \ale{Here, our goal is to increase the value of the t-conorm, i.e., the maximum value. Increasing other literals up to $\revis{\varphi}$ would require longer orange bars and bigger values for the L$_p$ norm. Figure~\ref{fig:Godel}(b) represents when multiple literals have the largest truth value. Here, only one should be increased\footnote{\ale{In our experiments, we choose randomly.}}. Finally, Figure~\ref{fig:Godel}(c) shows the \boosted\ vector for the G\"{o}del t-norm. Since the smallest truth value should be at least $\revis{\varphi}$, we simply ensure all truth values are at least $\revis{\varphi}$.}

Our results are closely related to that of \cite{giunchigliaMultiLabelClassificationNeural2021}, which considers hard constraints, i.e., $\revis{\varphi}=1$. In the hierarchical multi-label classification setting, the authors introduce an output layer that ensures predictions satisfy a set of hierarchy constraints. This layer corresponds to applications of the minimal \boost\ function for the G\"{o}del implication with $\revis{R_G}=1$. Furthermore, \cite{giunchigliaMultiLabelClassificationNeural2021} introduces C-HMCNN(h). This method considers an output layer that ensures predictions satisfy background knowledge expressed in a stratified normal logic program. The authors introduce an iterative algorithm that computes the minimal solution for such programs. This algorithm is related to that of ILR in Section \ref{sec:ILR}. However, their formalization differs somewhat from ours, and future work could study whether these results also hold for our formalization of minimal \boost\ functions and if they can be extended to any value of $\revis{\varphi}$. Finally, \cite{giunchigliaMultiLabelClassificationNeural2021} introduces a loss function that compensates for gradient bias introduced by the constrained output layer. 

\subsubsection{\luk\ t-norm}
\label{seq:lukasiewicz}
In this section, we derive minimal \boost\ functions for the \luk\ t-norm and t-conorm, for the family of $p$-norms. We will start using the following notation here: $\truth^\uparrow$ refers to the truth values $\truths_i$ sorted in ascending order, while $\truth^\downarrow$ refers to the truth values sorted in descending order.
% \begin{equation}
%     \ell=\sum_{i=1}^n \vert\boosts_i\vert^p + \lambda(\min(\sum_{i=1}^n \truths_i + \boosts_i, 1) - \weight)
% \end{equation}

% Combining Propositions \ref{prop:dual-t-conorm}, \ref{prop:luk-t-norm} and \ref{prop:luk-t-conorm}, we find the minimal \boost\ function for the  \luk\ t-norm for $\revis{T_L} \in [0, \max(\sum_{i\in F} \truths_i - \vert F \vert + 1, 0)]$: 
% \todo{Maybe a better condition would be to compare if $\minboostv_{M^ *} < 1-\truths_i$}

\begin{proposition}
 Let $\revis{T_L}\in[0, \max(\|\truthc\|_1 - (m - 1), 0)]$ and define $\lukincrease_K=\frac{\revis{T_L}+ m + K -1-\|\truthc\|_1 - \sum_{i=1}^ K\truths^\uparrow_{i}}{K}$. Let $K^ *$ be the largest integer $1\leq K\leq n$ such that $\lukincrease_{K}<1-\truths^ \uparrow_{K}$. Then the minimal \boost\ vector of the \luk\ t-norm is 
\begin{equation}
    \minboost_{T_L}(\truth, \revis{T_L})_i=\begin{cases}
        \truths_i + \lukincrease_{K^*} & \text{if } \revis{T_L} > T_L(\truth) \text{ and } \truths_i \leq \truths^\uparrow_{K^*}, \\
        1 & \text{if } \revis{T_L} > T_L(\truth) \text{ and } \truths_i > \truths^\uparrow_{K^*}, \\
        \truths_i - \frac{\max(\|\truth\|_1 + \|\truthc\|_1 + 1 - n - \revis{T_L}, 0)}{n} & \text{otherwise.}
    \end{cases} 
\end{equation}
% \todo{K=0 isn't defined because of the division...}
% \todo{I think there's a problem in the definition of $\minboostv_{K^ *}$ because it can never really be 0. Maybe it needs to be max 0? And what if w=1? Then the sum can exceed it, without it being a problem... This seems to have errors...}
Let $\revis{S_L} \in [\min(\|\truthc\|_1, 1), 1]$ and define $\lambda_K =  \frac{\|\truth\|_1 + \|\truthc\|_1  - \revis{S_L}}{K}$. Let  $K^*$ be the largest integer $1\leq K\leq n$ such that $\lambda_K < \truths^\downarrow_{K}$. Then the minimal \boost\ function of the \luk\ t-conorm is

\begin{equation}
    \minboost_{S_L}(\truth, \revis{S_L})_i=\begin{cases}
        \truths_i + \frac{\max(\revis{S_L}-\|\truth\|_1 - \|\truthc\|_1, 0)}{n} & \text{if } \revis{S_L} > S_L(\truth), \\
        % 0 & \text{if } \weight > S_L(\truth) \text{ and } \sum_{k=1}^n \truths_k \geq 1,\\
        \truths_i - \lambda_{K^*} & \text{if } \revis{S_L} < S_L(\truth) \text { and } \truths_i \geq \truths^\downarrow_{K^*}, \\
        0 & \text{otherwise.}
    \end{cases}
\end{equation} 
\end{proposition}
\begin{proof}
    This follows from Propositions \ref{prop:dual-t-conorm}, \ref{prop:luk-t-norm} and \ref{prop:luk-t-conorm}, see Appendix \ref{appendix:luk-tnorm} and \ref{appendix:luk-tconorm}.
\end{proof}
    % By assumption, for all  and so with similar reason as in the previous proof, we get the equality $\boosts_i=\boosts_j$ for all $i, j\in D$. The inequality constraint finds that $\sum_{k\in F} \truths_k + \sum_{k\in D} \truths_k + \boosts_k - (n-1) = \weight$.  
    % \begin{equation*}
    %     \sum_{k\in D} \boosts_k =\weight-\sum_{k=1}^n \truths_k + n-1,
    %     \quad
    %     % n\boosts_i &=\weight-\sum_k \truths_k \\
    %     \boosts_i = \frac{\weight-\sum_{k=1}^n \truths_k + n-1}{\vert D\vert}.
    % \end{equation*}

Although slightly obfuscated, these \boost\ functions simply increase each of the literals equally, while properly dealing with constraints on the truth values. \ale{We explain this using Figure~\ref{fig:Lukasiewciz_p2}, where the optimal solution corresponds to a vector that, from the original truth values $\truth$, is perpendicular to the contour line of the operator at the value $\revis{\varphi}$. Moreover, the figure also provides some intuition for our proofs. The stationary points of the Lagrangian correspond to the points where the constraint function (blue circumference) tangentially touches the contour line of the \boosted\ value (orange line).
% method of Lagrangian multipliers seeks to find solutions where the derivative of the t-norm (orange line) is equal to the derivative of the constraint (blue line), i.e., where the change applied is perpendicular to the 
}

\begin{figure}
    \includegraphics[width=\linewidth]{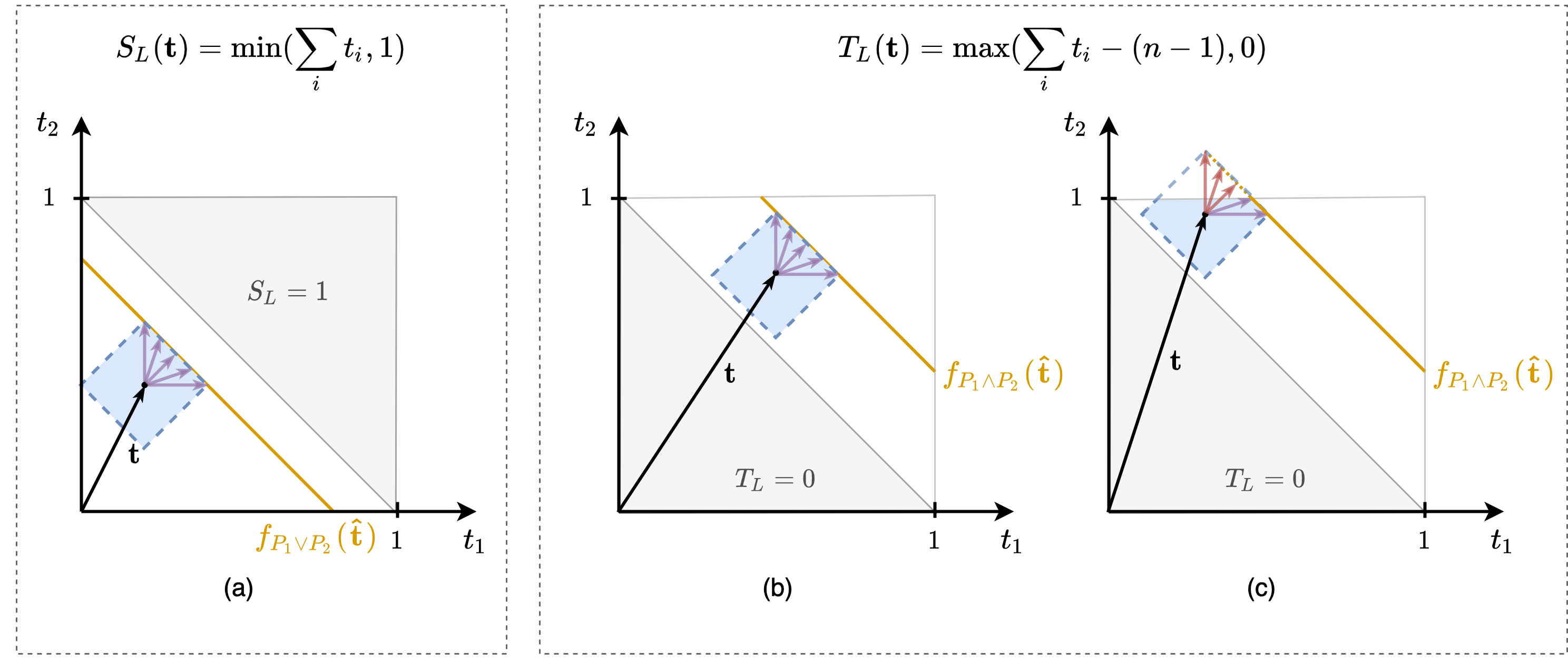}
    \caption{\ale{\luk\ minimal \boost\ functions. The orange line corresponds to the contour line of the $S_L$ and $T_L$ at the value $\revis{\varphi}$. Dotted blue circumference corresponds to a set of points at an equal distance from $\truth$. (a) t-conorm; (b) t-norm; (c) t-norm in the limit case.}}
    \label{fig:Lukasiewciz_p2}
\end{figure}

The change \ale{applied by the \boost\ function} is proportional to the \boost\ value $\revis{}$. Computing these \boost\ functions requires finding $K^*$, which can be done efficiently in log-linear time using a sort on the input truth values and a binary search.

The residuum of the \luk\ t-norm is equal to its S-implication formed using $S_L(1-a, c)$, and so its minimal \boost\ function can be found using Proposition \ref{prop:s-implication}. 

The \luk\ logic is unique in containing large convex and concave fragments \cite{gianniniConvexLogicFragment2019}. In particular, any CNF formula interpreted using the weak conjunction (Godel t-norm) and \luk\ t-conorm is concave, allowing for efficient maximization using a quadratic program of a slightly relaxed variant of the problem in Equation \ref{eq:optim-problem}. \cite{gianniniConvexLogicFragment2019} studies this property in a setting similar to ours in the context of collective classification. Future work could study using this convex fragment to find minimal \boost\ functions for more complex formulas. 

\subsubsection{Product t-norm}
To present the three basic t-norms together, we give the closed-form \boost\ function for the product t-norm with the L1 norm. Our proof is a special case of the general results on a large class of t-norms we will discuss in Section \ref{sec:general-analysis}. In particular, the product t-norm is a strict, Schur-concave t-norm with an additive generator. It is an example of a t-norm for which we can find a closed-form \boost\ function for the L1 norm using Propositions \ref{prop:additive-generator} and \ref{prop:dual-t-conorm}. First, we show the minimal \boost\ function for the product t-norm.

\begin{equation}
   \minboost_{T_P}(\truth, \revis{T_P})_i= \begin{cases}
       \sqrt[n-K^*]{\frac{\revis{T_P}}{\prod_{j=1}^{K^*}\truths_j^\downarrow \prod_{j=1}^mC_j}} & \text{if } T_P(\truth, \truthc) > \revis{T_P} \text{ and } \truths_i \leq \truths^\downarrow_{K^*+1}, \\
       \sqrt{\frac{\revis{T_P}}{\prod_{j\neq i}\truths_i^\downarrow \prod_{i=1}^mC_i}} & \text{if } T_P(\truth, \truthc) < \revis{T_P} \text{ and } i=\arg\min_{j=1}^n \truths_j,   \\
       \truths_i & \text{otherwise.}
   \end{cases} 
\end{equation}

Next, we present the result for the product t-conorm:
\begin{equation}
    \minboost_{S_P}(\truth, \revis{S_P})_i= \begin{cases}
        1-\sqrt{\frac{1-\revis{S_P}}{\prod_{j\neq i}1-\truths_i^\downarrow \prod_{i=1}^m1-C_i}} & \text{if } S_P(\truth, \truthc) < \revis{S_P} \text{ and } i=\arg\min_{j=1}^n \truths_j,   \\
        1-\sqrt[n-K^*]{\frac{1-\revis{S_P}}{\prod_{j=1}^{K^*}1-\truths_j^\downarrow \prod_{j=1}^m1-C_j}} & \text{if } S_P(\truth, \truthc) > \revis{S_P} \text{ and } \truths_i \leq \truths^\downarrow_{K^*+1}, \\
        \truths_i & \text{otherwise.}
    \end{cases} 
 \end{equation}

This \boosted\ function increases all the literals smaller than a certain threshold up to the threshold itself, where we assume $\revis{T_P}$ is greater than the initial truth value. In fact, like the other t-norms in the class discussed in Section \ref{sec:general-analysis}, it is similar to the G\"{o}del t-norm in that it increases all literals above some threshold to the same value. Similarly, the \boost\ function for the t-conorm increases the highest literal. Figure~\ref{fig:Product} gives an intuition behind this behaviour.
 
 \begin{figure}
    \includegraphics[width=\linewidth]{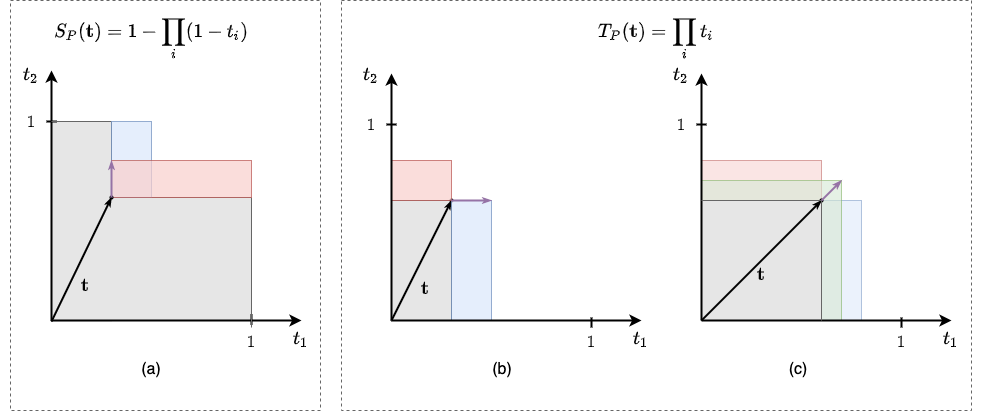}
    \caption{\ale{Product minimal \boost\ functions. The grey areas represent the truth value of the operator associated with the initial vector $\truth$. Red and blue areas represent the \boosted\ values when increasing a single literal.
    (a) t-conorm; (b) t-norm; (c) t-norm when multiple literals have same truth value. The green area represents the improvement obtained by increasing both literals equally.}}
    \label{fig:Product}
\end{figure}

Finally, the residuum minimal \boost\ function can be found with $\minboost_{I_P}(\truths_1, \truths_2, \revis{I_P}) = [\truths_1, \frac{\revis{I_P}}{\truths_1}]^\top$. 

We also studied the minimal \boost\ function for the L2 norm, but concluded that the result is a $2n$th degree polynomial with no simple closed-form solutions. For details, see Appendix \ref{Appendix:product-l2}. 

% \emile{I'm not sure if it's relevant to show this image yet.}
% \begin{figure}[H]
%     \includegraphics[width=\linewidth]{L1.png}
%     \caption{....}
%     \label{fig:Lukasiewciz_p1}
% \end{figure}

\section{A general class of t-norms with analytical minimal \boost\ functions}
\label{sec:class-analysis}
In this section, we will introduce and discuss a general class of t-norms that have analytic solutions to the problem in Equation \ref{eq:optim-problem} in order to find their corresponding minimal \boost\ functions. We can find those for the t-norm, t-conorm, and the residuum. 
\subsection{Background on t-norms}
To be able to properly discuss this class of t-norms, we first have to provide some more background on the theory of t-norms. 
\subsubsection{Additive generators}
The study of t-norms frequently involves the study of their \emph{additive generator} \cite{klementTriangularNorms2013,klementTriangularNormsPosition2004}, which are univariate functions that can be used to construct t-norms, t-conorms and residuums. 

\begin{definition}
    A function $\add: [0, 1]\rightarrow [0, \infty]$ such that $\add(1)=0$ is an \emph{additive generator} if it is strictly decreasing, right-continuous at 0, and if for all $t_1, t_2\in [0, 1]$, $\add(t_1)+\add(t_2)$ is either in the range of $\add$ or in $[\add(0^{+}), \infty ]$. 
\end{definition}
\begin{theorem}
    If $\add$ is an additive generator, then the function $T: [0, 1]^n\rightarrow [0, 1]$ defined as 
    \begin{equation}
        \label{eq:additive-generator}
        T(\truth) = \add^{-1}(\min(\add(0^+), \sum_{i=1}^n \add(\truths_i)))
    \end{equation}
    is a t-norm.
\end{theorem}
Using Equation \ref{eq:additive-generator}, the function $g$ acts like an invertible function. It transforms truth values into a new space that can be seen as measuring `untruthfulness'. $\sum_{i=1}^n \add(\truths_i)$ can be seen as a measure of `untruth' of the conjunction. T-norms constructed in this way are necessarily Archimedean, and each continuous Archimedean t-norm has an additive generator. $T_P$, $T_L$ and $T_D$ have an additive generator, but $T_G$ and $T_N$ do not. Furthermore, if $\add(0^+)=\infty$, $T$ is strict and we find $T(\truth)=\add^{-1}(\sum_{i=1}^n \add(\truths_i))$. 
The residuum constructed from continuous t-norms with an additive generator can be computed using $\add^{-1}(\max(\add(c)-\add(a), 0))$ \cite{jayaramFuzzyImplications2008}. 

\subsubsection{Schur-concave t-norms}
We will frequently consider the class of Schur-concave t-norms, with their dual t-conorms and residuums formed from these Schur-concave t-norms. We denote with $\truth^\downarrow$ the truth vector $\truth$ sorted in descending order, and with $\truth^ \uparrow$ as $\truth$ sorted in ascending order. \todo{Do we need to define the latter one?}
\begin{definition}
    \label{def:schur-concave}
    A vector $\truth\in\mathbb{R}^n$ is said to \emph{majorize} another vector $\truthalt \in\mathbb{R}^ n$, denoted $\truth\succ\truthalt$, if $\sum_{i=1}^n \truths_{i}=\sum_{i=1}^n {\truthsalt_i}$ and if for each $i\in\{1, ..., n\}$ it holds that $\sum_{j=1}^i \truths_{ j}^\downarrow \geq \sum_{j=1}^i \truthsalt_{j}^\downarrow$. 
\end{definition}
\begin{definition}
    A function $[0, 1]^ n\rightarrow [0,1]$ is called \emph{Schur-convex} if for all $\truth, \truthalt\in [0, 1]^n$, $\truth\succ \truthalt$ implies that $f(\truth) \geq f(\truthalt)$. Similarly, a \emph{Schur-concave} function has that $\truth \succ \truthalt$ implies that $f(\truth) \leq f(\truthalt)$. 
\end{definition}
The dual t-conorm of a Schur-concave t-norm is Schur-convex. The three basic and continuous t-norms $T_G$, $T_P$ and $T_L$ are Schur-concave. There are also non-continuous Schur-concave t-norms, such as the Nilpotent minimum \cite{takaciSchurconcaveTriangularNorms2005,vankriekenAnalyzingDifferentiableFuzzy2022}. The drastic t-norm is an example of a t-norm that is not Schur-concave \cite{takaciSchurconcaveTriangularNorms2005}. This class includes all quasiconcave t-norms since all symmetric quasiconcave functions are also Schur-concave \cite[p98 C.3]{marshallSchurConvexFunctions2011}. Therefore, this class constitutes a significant class of relevant t-norms. For a more precise characterization of Schur-concave t-norms, see \cite{takaciSchurconcaveTriangularNorms2005,alsinaSchurConcaveTNormsTriangle1984}.

\subsection{Minimal \boost\ functions for Schur-concave t-norms}
We now have the background to discuss several useful and interesting results on Schur-concave t-norms. First, we present two results that characterize Schur-concave minimal \boost\ functions.  We use the notion of \say{strictly cone-increasing} functions here that is discussed in Appendix \ref{appendix:cone-monotonicity}.
\begin{theorem}
    \label{theorem:schur-concave-t-norm}
    Let $T$ be a Schur-concave t-norm that is strictly cone-increasing at $\revis{T}$ and let $\|\cdot \|$ be a strict norm. Then there is a minimal \boosted\ vector $\minboostv$ for $\truth$ and $\revis{T}$ such that whenever $\truths_i> \truths_j$, then $\minboosts_i - \truths_i\leq \minboosts_j - \truths_j$. 
\end{theorem}

For proof, see Appendix \ref{appendix:boost-schur-tnorm}. We note that we can make this argument in the other direction to show that any Schur-convex t-conorm will have a minimal \boosted\ vector such that $\truths_i> \truths_j$ implies $\minboosts_i \geq \minboosts_j$. Furthermore, if we know that a t-norm has a unique minimal \boost\ function, we can use this theorem to infer a useful ordering on how it changes the truth values. 

Next, we will consider the L1 norm $\sum_{i=1}^n\vert \boosts_i - \truths_i\vert$, for which we can find general solutions for the t-norm, t-conorm and R-implication when the t-norm is Schur-concave. 
\begin{proposition}
    Let $\truth\in [0, 1]^n$ and let $T$ be a Schur-concave t-norm that is strictly cone-increasing at $\revis{T}\in [T(\truth, \truthc), \revismax{T}]$. Then there is a value $\lambda\in [0, 1]$ such that the vector $\minboostv$,
    \begin{equation}
        \label{eq:minboost-t-norm-schur-concave}
        \minboosts_i = \begin{cases}
            \lambda, & \text{if } \truths_i < \lambda, \\
            \truths_i, & \text{otherwise,}
        \end{cases}
    \end{equation}
    is a minimal \boosted\ vector for $T$ and the L1 norm at $\truth$ and $\revis{T}$.
\end{proposition}

For proof, see Appendix \ref{appendix:closed-form-additive}. We found this result rather surprising: It is optimal for a large class of t-norms and the L1 norm to increase the lower truth values to some value $\lambda$. In this sense, these solutions are very similar to that of the G\"{o}del \boost\ functions. The value of $\lambda$ depends on the choice of t-norm and $T(\minboostv, \truthc)$ is a non-decreasing function of $\lambda$. We show in Section \ref{sec:additive-generators} how to compute these. 
% We note that $T(\minboostv)$ is a non-decreasing function of $\lambda$.

We have a similar result, proof in the end of Appendix \ref{appendix:closed-form-additive}, for the \boost\ functions of Schur-convex t-conorms. This proposition shows that, under the L1 norm, it is optimal to increase only the largest literal, just like with the G\"{o}del t-norm.

\begin{proposition}
    Let $\truth\in [0, 1]^n$ and let $S$ be a Schur-convex t-conorm that is strictly cone-increasing at $\revis{S}\in [S(\truth, \truthc), 1]$. Then there is a value $\lambda\in [0, 1]$ such that the vector $\minboostv$,
    \begin{equation}
        \minboosts_i = \begin{cases}
            \lambda & \text{if } i={\arg\max}_{i\in D}\truths_i, \\
            \truths_i, & \text{otherwise,}
        \end{cases}
    \end{equation}
    is a minimal \boosted\ vector for $S$ and the L1 norm at $\truth$ and $\revis{S}$.
\end{proposition}

\subsection{Closed forms using Additive Generators}
\label{sec:additive-generators}
Where the previous section gives general results on the form or \say{shape} of minimal \boost\ functions for t-norms and t-conorms under the L1 norm, we still need to figure out what the value of $\lambda$ is for a particular $\revis{\varphi}$. Luckily, additive generators will do the job here. 
\begin{proposition}
    \label{prop:additive-generator}
    Let $T$ be a Schur-concave t-norm with additive generator $g$ and let $0<\revis{T}\in [T(\truth, \truthc), \revismax{T}]$. 
    Let $K\in \{0, ..., n-1\}$ denote the number of truth values such that $\minboosts_i=\truths_i$ in Equation \ref{eq:minboost-t-norm-schur-concave}.
    % Let $N\subseteq \{1, ..., n\}$ be the set of indices $i$ such that the $\minboostv$ from Equation \ref{eq:minboost-t-norm-schur-concave} has $\minboosts_i=\truths_i$. 
    Then using

    \begin{equation}
        \lambda_K = g^ {-1}\left(\frac{1}{n-K}\left(g(\revis{T}) -\sum_{i=1}^K g(\truths^\downarrow_i) - \sum_{i=1}^m g(C_i)\right)\right)
    \end{equation}

    in Equation $\ref{eq:minboost-t-norm-schur-concave}$ gives $T(\minboostv, \truthc)=\revis{T}$ if $\minboostv\in [0, 1]^n$. 
\end{proposition}

See Appendix \ref{appendix:closed-form-additive} for a proof. $g(\revis{T})$ can be seen as the `untruth'-value in $g$-space that $\minboostv$ should attain. Since we have $n-K$ truth values that we can move freely, we need to make sure that their `untruth'-value in $g$-space is $g(\revis{T})/(n-K)$. However, we also need to handle the truth values we cannot change freely, which is why those are subtracted from $g(\revis{T})$. 

We should note that this does not yet give a procedure for computing the correct $K\in \{0, ..., n-1\}$. The intuition here is that we should find an $K$ such that $\truths_i \geq \lambda_K$ for the $K$ largest values, and $\truths_i < \lambda_K$ for the remaining $n-K$. Like with computing the $K^*$ for the \boost\ function for the \luk\ t-norm (Section \ref{seq:lukasiewicz}), we can do this in logarithmic time after sorting $\truth$, but we choose to compute $\lambda_K$ for each $K\in \{0, n-1\}$ in parallel. 

We can similarly find a closed form for the t-conorms: 
\begin{align}
    \lambda &= 1 - g^{-1}\left(g(1-\revis{S}) - \sum_{i\neq j}g(1-\truths_i) -  \sum_{i=1}^m g(1-C_i)\right)
\end{align}

\emile{I could not figure out how to prove this when not assuming $T$ is strict...}
\begin{proposition}
    Let $\truths_1, \truths_2\in [0,1]$ and let $T$ be a strict Schur-concave t-norm with additive generator $g$. Consider its residuum $R(\truths_1, \truths_2)=\sup \{z\vert T(\truths_1, z)\leq \truths_2\}$ that is strictly cone-increasing at $0<\revis{R}\in [R(\truths_1, \truths_2), \revismax{R}]$. Then there is a value $\lambda\in [0, 1]$ such that $\minboostv=[\truths_1, g^{-1}(g(\revis{R}) + g(\truths_1))]^\top$ is a minimal \boosted\ vector for $R$ and the L1 norm at $\truth$ and $\weight$. 
\end{proposition}

% From the proof, it should be clear that the value $\lambda$ can be computed with 
% \begin{equation}
%     \lambda = g^{-1}(g(\revis{R}) + g(\truths_1))
% \end{equation}

Here, we find that for this class of residuums, increasing the consequent (the second argument of the implication) is minimal for the L1 norm. This update reflects modus ponens reasoning: When the antecedent is true, increase the consequent. As we have argued in \cite{vankriekenAnalyzingDifferentiableFuzzy2022}, this could cause issues in many machine learning setups: Consider the modus tollens correction instead decreases the antecedent. For common-sense knowledge, this is more likely to reflect the true state of the world.

\section{Experiments}

We performed experiments on two tasks. The first one does not involve learning. Instead, we aim to solve SAT problems. This experiment allows assessing whether ILR can enforce complex and unstructured knowledge. The second experiment is on the MNIST Addition task~\cite{manhaeveDeepProbLogNeuralProbabilistic2018} to test ILR in a neuro-symbolic setting and assess its ability to learn from data.

\subsection{Experiments on 3SAT problems}
With this experiments, our goal is to find out how quickly ILR finds a \boosted\ vector and how minimal this vector is. We test this on formulas of varying complexity to analyze for what problems each algorithm performs well\footnote{Code available at \href{https://github.com/DanieleAlessandro/IterativeLocalRefinement}{https://github.com/DanieleAlessandro/IterativeLocalRefinement}}.

\subsubsection{Setup}
We perform experiments on SATLIB \cite{hoosSATLIBOnlineResource2000}, a library of randomly generated 3SAT problems. 3SAT problems are formulas in the form $\bigwedge_{i=1}^c \bigvee_{j=1}^3 l_{ij}$, where $l_{ij}$ is a literal that is either $P_k$ or $\neg P_k$ and where $P_k\in\{P_1, ..., P_n\}$ is an input proposition. In particular, we consider uf20-91 of satisfiable 3SAT problems with $n=20$ propositions and $c=91$ disjunctive clauses. For this, we select \ale{the \boosted\ value $\revis{\varphi}$ to be 1. We also experiment with $\revis{\varphi}\in\{0.3, 0.5\}$ in Appendix \ref{appendix:additional-experiments}}. We uniformly generate initial truth values for the propositions $\truth\in [0,1]^d$~\footnote{\ale{Each run used the same initial value for each algorithm to have a fair comparison.}}. To allow experimenting with formulas of varying complexity, we \ale{introduce a simplified version of the task which uses only the first 20 clauses}.
% choose to vary the amount of active clauses by choosing only the first $c\in \{10, 20, 30, 40, 50, 60, 70, 80, 91\}$\ale{maybe here we can remove a lot of them.. Otherwise, there are too many results to show} clauses.

We compare ILR with a gradient descent baseline described in Section \ref{sec:gradient-descent} using three metrics. The first is speed: How many iterations does it take for each algorithm to converge? Since both algorithms have similar computational complexities, we will use the number of iterations for this. The second is satisfaction: Is the algorithm able to find a solution with truth value $\revis{\varphi}$?
% For this we take the \ale{RMSE (I used the mean in the end since we consider only the target 1. Plots are clearer)} between the truth value after convergence and $\revis{\varphi}$ (that is, the distance in L2).
Finally, we consider minimality: How close to the original prediction is the \boosted\ vector $\boostv$? \ale{Note that the \boost\ function for the product logic is only optimal for the L1 norm, while
for G\"{o}del and \luk, the \boost\ function is optimal for all L$_p$ norms, including L1. Moreover, the results of L1 and L2 are very similar. Therefore, we use the L1 as a metric for minimality for each t-norm.}
% e use the L2 norm for G\"{o}del and \luk, and the L1 norm for the product t-norm since the minimal \boost\ function for the product t-norm is only minimal in the L1 norm. Accordingly, we use L2 regularization in the gradient descent algorithm for G\"{o}del and \luk, and L1 regularization for the product. 

\subsubsection{Gradient descent baseline}
\label{sec:gradient-descent}
We compare ILR to gradient descent with the following loss function 
\begin{equation}
    \mathcal{L}(\hat{\bz}, \truth, \revis{\varphi})= \|\op_\varphi(\sigma(\hat{\bz})) - \revis{\varphi}\|_2 + \alpha \|\sigma(\hat{\bz}) - \truth\|_p.
\end{equation}
Here $\boostv = \sigma(\hat{\bz})$ is a real-valued vector $\hat{\bz}\in \mathbb{R}^n$ transformed to $\boostv\in [0, 1]^n$ using the sigmoid function $\sigma$ to ensure the values of $\boostv$ remain in $[0,1]^n$ during gradient descent. The first term minimizes the distance between the current truth value of the formula $\varphi$ and the \boost\ value, while the second term is a regularization term that minimizes the distance between the \boosted\ vector and the original truth value $\truth$ in the L$p$ norm. $\alpha$ is a hyperparameter that trades off the importance of this regularization term. 

This method for finding \boosted\ vectors is very similar to the collective classification method introduced in SBR \cite{diligentiSemanticbasedRegularizationLearning2017,roychowdhuryRegularizingDeepNetworks2021}. The main difference is in the L$p$ norms chosen, as we use squared error for the first term instead of the L1 norm. 
Gradient descent is a steepest descent method that take steps minimizing the L2 norm. Therefore, it can also be seen as a method for finding minimal \boost\ functions given the L2 norm. The corresponding steepest descent method for the L1 norm is the coordinate descent algorithm. Future work could compare how coordinate descent performs for finding minimal \boost\ functions for the L1 norm. We suspect it will be much slower than gradient descent-based methods as it can only change a single truth value each iteration.  

We found that ADAM \cite{kingmaAdamMethodStochastic2017} significantly outperformed standard gradient descent in all metrics, and we chose to use it throughout our experiments. Furthermore, inspired by the analysis of the derivatives of aggregation operators in \cite{vankriekenAnalyzingDifferentiableFuzzy2022}, we slightly change the formulation of the loss function for the \luk\ t-norm and product t-norm.
The \luk\ t-norm will have precisely zero gradients for most of its domain. Therefore, we remove the $\max$ operator when evaluating the $\bigwedge$ in the SAT formula so it has nonzero gradients. For the product t-norm, the gradient will also approach 0 because of the large set of numbers between $[0, 1]$ that it multplies. As suggested by \cite{vankriekenAnalyzingDifferentiableFuzzy2022}, we instead optimize the logarithm of the product t-norm:
\begin{equation*}
    \mathcal{L}_P(\hat{\bz}, \truth, \revis{\varphi})= \| \sum_{i=1}^c \log f_{\bigvee_{j=1}^3}(\sigma(\truth)) - \log \revis{\varphi}\|_2 + \alpha \|\sigma(\hat{\bz}) - \truth\|_1.
\end{equation*}

\subsubsection{Results}
\label{sec:results-sat}
\begin{figure}
    \includegraphics[width=\linewidth]{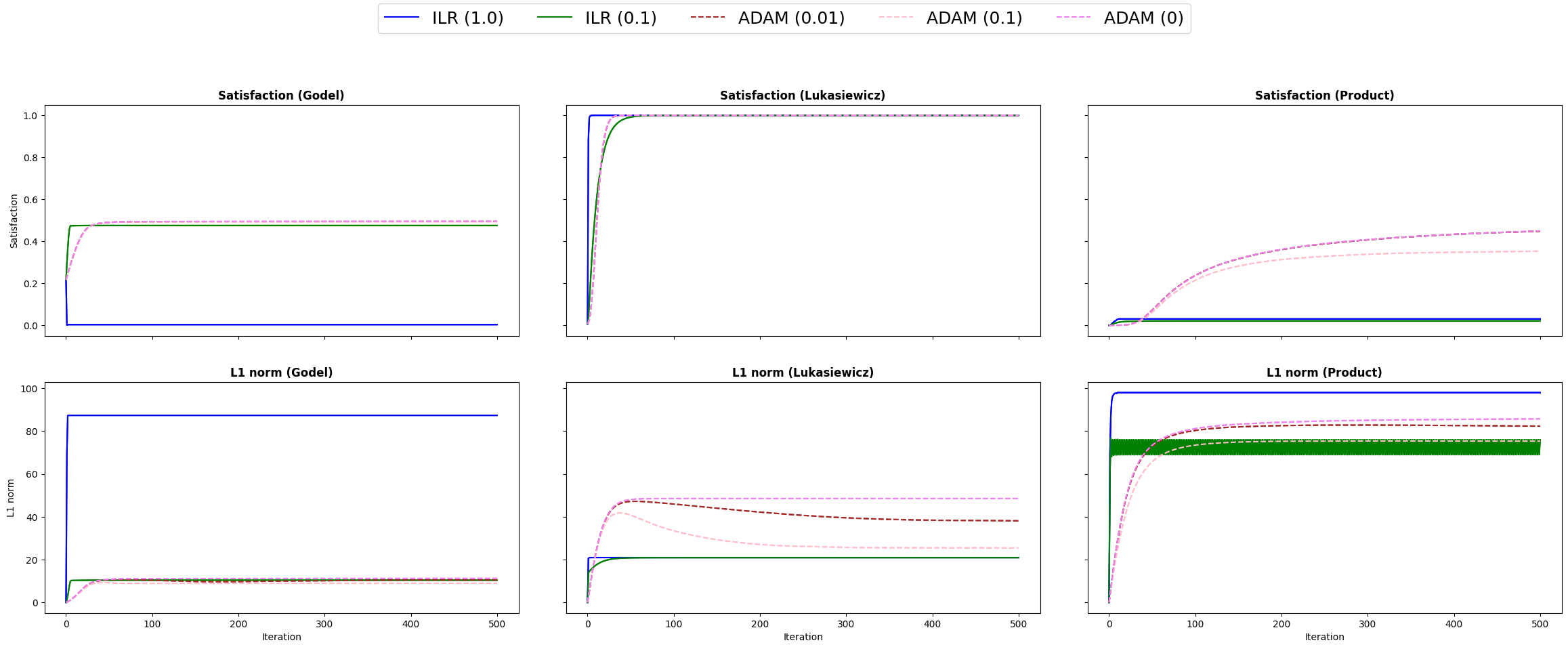}
    \caption{\ale{Comparison of ILR with ADAM on uf20-91 in SATLIB. \Boosted\ value 1.0.}. The x axis corresponds to the number of iterations, while the y axis is the value of $\revis{\varphi}$ in the first row of the grid, and the L1 norm in the second row.}
    \label{fig:results_91}
\end{figure}

\ale{
In Figure~\ref{fig:results_91} we show the results obtained by ILR and ADAM on the three t-norms (one for each column of the grid).} %For ILR we use both the mean and max to combine truth values as introduced in Section~\ref{sec:ILR}. }
\ale{
We observe that ILR with schedule parameter $\alpha=0.1$ has a smoother plot compared to ILR with $\alpha=1.0$, which converges faster: In our experiments, the number of steps until convergence was always between 2 and 5. For both values of the scheduling parameters, ILR outperforms ADAM in terms of convergence speed.}

\ale{
When comparing satisfaction and minimality, the behavior differs based on the t-norm used. In the case of \luk, all methods find feasible solutions to the optimization problem. Furthermore, in terms of minimality (i.e., L1 norm), ILR finds better solutions than ADAM.}

\ale{
For the G\"{o}del logic no method is capable of reaching a feasible solution. Here, ILR with schedule parameter $\alpha=1$ performs very poorly, obtaining worse solutions than the original truth values. On the other hand, with $\alpha = 0.1$, it performs as well as ADAM for both metrics but with faster convergence.
}

\ale{
Finally, for the product logic, ILR fails to increase the satisfaction of the formula to the \boosted\ value. However, ADAM can find much better solutions, getting the average truth value to around 0.5. Still, it is far from reaching a feasible solution. Nonetheless, we recommend using ADAM for complicated formulas in the product logic.
}

\ale{However, we argue that in the context of Neural-Symbolic Integration, the provided knowledge is usually relatively easy to satisfy. } \emi{With 91 clauses, there are few satisfying solutions in this space of $2^ {21}$ possible binary solutions. However, background knowledge usually does not constrain the space of possible solutions as heavily as this.} \ale{For this reason, we propose a simplified formula, where we only use 20 out of 91 clauses. Figure~\ref{fig:results_20} shows the results for this setting.} \emi{We see that ILR with no scheduling ($\alpha=1$) finds feasible solutions for all t-norms. ILR finds solutions for the G\"{o}del t-norm where ADAM cannot find any, while for \luk\ and product, it finds solutions in much fewer iterations and with a lower L1 norm. Hence, we argue that for knowledge bases that are less constraining, ILR without scheduling is the best choice.}

\begin{figure}
    \includegraphics[width=\linewidth]{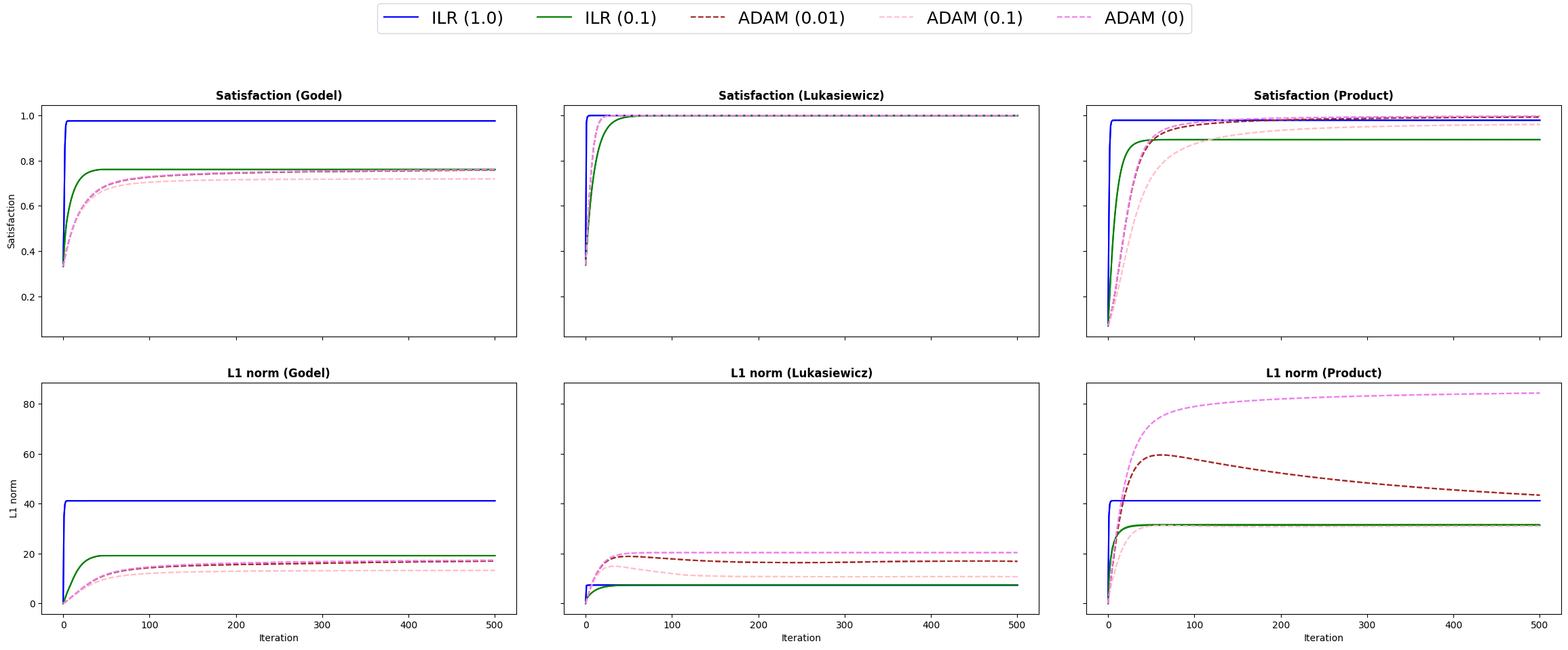}
    \caption{\ale{Comparison of ILR with ADAM on the uf20-91 with 20 clauses. Target value 1.0.}}
    \label{fig:results_20}
\end{figure}

\subsection{Experiments on MNIST Addition}
\label{sec:MNIST_exp}
The experiments on the SATLIB benchmark show how well ILR can enforce knowledge in highly constrained settings. However, as already mentioned, in neuro-symbolic AI, the background knowledge is typically much simpler. SAT benchmarks often only have a few solutions, heavily limiting what predictions the neural network can make.
%Injecting this type of knowledge on a neural network would mean that the model returns a constant, which would make useless the usage of a neural network.}  
Moreover, previous experiments only tested ILR where initial truth vectors are random, and we did not have any neural networks or learning. 

To evaluate the performance of ILR in neuro-symbolic settings, we implemented the architecture of Figure~\ref{fig:ILR_MNIST}. Here, the task is to learn a classifier for handwritten digits while only receiving supervision on the sums of pairs of digits.

\subsubsection{Setup}

We follow the architecture of Figure~\ref{fig:ILR_MNIST}. We use the neural network proposed by~\cite{manhaeveDeepProbLogNeuralProbabilistic2018}, which is a network composed of two convolutional layers, followed by a MaxPool layer, followed by a fully connected layer with ReLU activation function and a fully connected layer with softmax activation. We use the G\"{o}del t-norm and corresponding minimal \boost\ functions. Note that G\"{o}del implication can only increase the consequent and can never decrease the antecedents. For this reason, ILR converges in a single step.

We set both $\alpha$ and target value $\hat{t}$ to one, meaning that we ask ILR to make the entire formula completely satisfied in one step. We use the ADAM optimizer and a learning rate of 0.01, with the cross-entropy loss function. However, since the outputs of the ILR step do not sum to one, we cannot directly apply it to the \boosted\ vector ILR computes. To overcome this issue, we add a logarithm followed by a softmax as the last layers of the model. If the sum of the \boosted\ vector is one, the composition of the logarithm and softmax functions corresponds to the identity function. Moreover, these two layers are monotonic increasing functions and preserve the order of the \boosted\ vector.

We use the dataset defined in~\cite{manhaeveDeepProbLogNeuralProbabilistic2018} with 30000 samples, and also run the experiment using only 10\% of the dataset (3000 samples). We run ILR for 5 epochs on the complete dataset, and 30 epochs on the small one. We repeat this experiment 10 times.
We are interested in the accuracy obtained in the test set for the addition task.
We ran the experiments on a MacBook Pro (2016) with a 3,3 GHz Dual-Core Intel Core i7.

\subsubsection{Results}

\begin{table}[]
    \centering
    \begin{tabular}{l|l|l}
         & 30000 & 3000 \\
         \hline
        DeepProblog~\cite{manhaeveDeepProbLogNeuralProbabilistic2018} & 97.20 $\pm$ 0.45 & 92.18 $\pm$ 1.57 \\
        LTN~\cite{badreddineLogicTensorNetworks2022} & 96.78 $\pm$ 0.5 & 92.15 $\pm$ 0.75 \\
        ILR & 96.67 $\pm$ 0.45 & 93.38 $\pm$ 1.70\\
    \end{tabular}
    \caption{Results on the MNIST addition task. We report the accuracy of predicting the sum (in \%) on the test set with 30000 and 3000 samples. DeepProbLog results are taken from~\cite{badreddineLogicTensorNetworks2022}. LTN results have been obtained by replicating the experiments of~\cite{badreddineLogicTensorNetworks2022}.}
    \label{tab:results}
\end{table}

ILR can efficiently learn to predict the sum, reaching results similar to the state of the art, requiring on average 30 seconds per epoch. However, sometimes ILR got stuck in a local minimum during training, where the accuracy reached was close to 50\%. It is worth noticing that LTN suffers from the same problem~\cite{badreddineLogicTensorNetworks2022}, with results strongly dependent on the initialization of the parameters. 
% However, no explanations on this problem have been presented so far. 
To better understand this local minimum, we analyzed the confusion matrix.
\begin{figure}
    \centering
    \includegraphics[width=7cm]{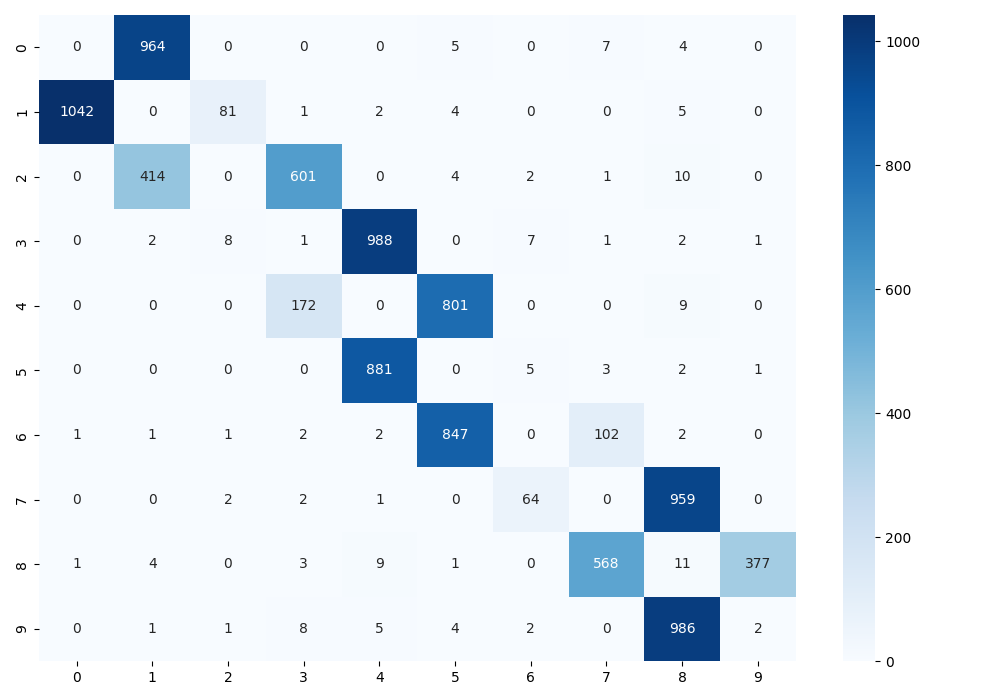}
    \caption{Confusion matrix on the MNIST classification for a local minimum}
    \label{fig:confusion}
\end{figure}
Figure~\ref{fig:confusion} shows one of the confusion matrices for a model stuck in the local minimum: the CNN recognizes each digit either as the correct digit minus one or plus one. Then, our model obtains the correct prediction in close to 50\% of the cases. For example, suppose the digits are a 3 and a 5. The 3 is classified either as either a 2 or a 4, while the 5 is classified as a 4 or a 6. If the model predicts 2 and 6 or 4 and 4, it returns the correct sum (8), otherwise, it does not. We believe that in these local minima there is no way for the model to change the digit predictions without increasing the loss, and the model remains stuck in the local minimum.

Table~\ref{tab:results} shows the results in terms of accuracy of ILR, LTN~\cite{badreddineLogicTensorNetworks2022} and DeepProblog~\cite{manhaeveDeepProbLogNeuralProbabilistic2018}. To calculate the accuracy, we follow~\cite{badreddineLogicTensorNetworks2022} and select only the models that do not stop in the local minimum. Notice that this problem is very rare for ILR (once every 30 runs) and happens more frequently with LTN (once every 5 runs).

\section{Conclusion and Future Work}
We analytically studied a large class of minimal fuzzy \boost\ functions. We used \boost\ functions to construct ILR, an efficient algorithm for general formulas. Another benefit of these analytical results is to get a good intuition into what kind of corrections are done by each t-norm.
In our experimental evaluation of this algorithm, we found that \emi{our algorithm converges much faster and often finds better solutions than the baseline ADAM, especially for problems that are less constraining. However, for complicated formulas and the product logic, we conclude ADAM finds better results}. Finally, we assess ILR on the MNIST Addition task and show it can be combined with a neural network, providing results similar to two of the most prominent methods for neuro-symbolic AI.

There is a lot of opportunity for future work on \boost\ functions. We will study how the \boost\ functions induced by different t-norms perform in practical neuro-symbolic integration settings. On the theoretical side, possible future work could be considering analytical \boost\ functions for certain classes of complex formulas. Furthermore, there are many classes of t-norms and norms for which finding analytical \boost\ functions is an open problem. Another promising avenue for research is designing specialized loss functions that handle biases in the gradients arising from combining constrained output layers with cross-entropy loss functions \cite{giunchigliaMultiLabelClassificationNeural2021}. We also want to highlight the possibility of extending the work on fuzzy \boost\ functions to probabilistic \boost\ functions, using a notion of minimality such as the KL-divergence.

% \section*{Declarations}
% \begin{description}
% \item[Funding:] Alessandro Daniele and Emile van Krieken are involved in a HumaneAI Microproject.
% % funded by a HumaneAI Microproject. 
% HumaneAI received funding from the European Union’s Horizon 2020 research and innovation program under grant agreement No 761758.
% \item[Conflicts of interest/Competing interests:] The authors have no conflicts of interest to declare that are relevant to the content of this article %(include appropriate disclosures)
% \item[Ethics approval:] We declare that our manuscript follows the ethics rules provided in \href{https://www.springer.com/gp/editorial-policies/ethical-responsibilities-of-authors}{https://www.springer.com/gp/editorial-policies/ethical-responsibilities-of-authors}. \tododeclarations{Is it ok?} %(include appropriate approvals or waivers)
% \item[Consent to participate:] Not applicable %(include appropriate consent statements)
% \item[Consent for publication:] Not applicable %(appropriate statements regarding publishing an individual’s data or image)
% \item[Availability of data and material:] Data used in this work can be downloaded from \href{https://www.cs.ubc.ca/~hoos/SATLIB/benchm.html}{https://www.cs.ubc.ca/~hoos/SATLIB/benchm.html} %(data transparency)
% \item[Code availability:] The code is available as an open source project on GitHub. \href{https://github.com/DanieleAlessandro/IterativeLocalRefinement}{https://github.com/DanieleAlessandro/IterativeLocalRefinement}
%  % (software application or custom code)
% \item[Authors' contributions:] Alessandro Daniele and Emile van Krieken: formal proofs, experiments, writing;
% Frank van Harmelen and Luciano Serafini: supervision, writing.
% \end{description}

\section*{Acknowledgements}
Alessandro Daniele and Emile van Krieken are involved in a HumaneAI Microproject.
% funded by a HumaneAI Microproject. 
HumaneAI received funding from the European Union’s Horizon 2020 research and innovation program under grant agreement No 761758.

% \begin{itemize}
%     \setlength\itemsep{0em}
%     \item Future work could consider analytical results with more complex formulas.
%     \item Also it is interesting to see if there is a class of formulas for which we can find exact boost functions for L1 (eg tree, but we can already prove this)
%     \item Compare with the benefit of convexity using the convex logic fragment. Maybe not needed to mention since we already discuss this in the section about \luk. 
%     \item Are there classes of fuzzy operators for which there are analytical solutions for L2? (Eg we can find them for Lukasiewciz and G\"{o}del, what about other ones?)
%     \item Create LRL using probabilistic logics (how to change probability minimally, maybe minimality can be measured using KL-divergence, as mentioned in related work)
% \end{itemize}

\bibliographystyle{abbrvnat}
\bibliography{references.bib}

\appendix

\section{Basic T-norms (Proofs)}
\subsection{G\"{o}del t-norm minimal \boosted\ function proofs}
\subsubsection{G\"{o}del t-norm}
\label{appendix:godel-t-norm}
\begin{proposition}
    \label{prop:godel-t-norm}
    The minimal \boost\ function of the G\"{o}del t-norm for $\revis{T_G} \in [T_G(\truth, \truthc), \min_{i=1}^m C_i]$ is
    \begin{equation}
        \minboost_{T_G}(\truth, \revis{T_G})_i=\begin{cases}
            \revis{T_G} & \text{if } \truths_i < \revis{T_G}, \\
            \truths_i & \text {otherwise}
        \end{cases}
    \end{equation} 
\end{proposition}
\begin{proof}
    Assume otherwise. Then there is a \boosted\ vector $\boostv$ for $T_G$, $\truth\in [0, 1]^n$ and $\revis{T_G} \in [T_G(\truth), \min_{i=1}^m]$ such that $\boostv\neq \minboostv$ while $\|\boostv - \truth\|_p < \|\minboostv - \truth\|_p$, where $\minboostv=\minboost_{T_G}(\truth, \revis{T_G})$. Since $T_G(\boostv)=\revis{T_G}$, for all $i\in \{1, ..., n\}$, $\boosts_i\geq \revis{T_G}$ and so necessarily for all $i$ such that $\truths_i< \revis{T_G}$, $\boosts_i \geq \revis{T_G}$. Since there is some $i$ such that $\boosts_i\neq \minboosts_i$, either $\truths_i < \revis{T_G}$ and then necessarily $\boosts_i > \minboosts_i$, or $\boosts_i \geq \revis{T_G}$ but $\boosts_i \neq \minboosts_i=\truths_i$. In either case, since $\|\cdot \|_p$ is strictly convex in each argument with minimum at $\truth$, $\|\boostv - \truth\|_p > \|\minboostv - \truth\|_p$, hence $\boostv$ could not have smaller norm. 
\end{proof}

\subsubsection{G\"{o}del t-conorm}
\label{appendix:godel-t-conorm}
A derivation for increasing the G\"{o}del t-conorm was first presented in \cite{danieleKnowledgeEnhancedNeural2019} and is adapted to our notation here:
\begin{proposition}
    \label{prop:godel-t-conorm}
    The minimal \boost\ function of the G\"{o}del t-conorm for $\revis{S_G} \in [S_G(\truth, \truthc), 1]$ is 
    \begin{equation}
        \minboost_{S_G}(\truth, \revis{S_G})_i=\begin{cases}
            \revis{S_G} & \text{if } i=\arg\max_{j=1}^n \truths_j \\
            \truths_i & \text {otherwise.}
        \end{cases}
    \end{equation}
\end{proposition}

\subsubsection{G\"{o}del Implication}
\label{appendix:godel-implication}
We next present a proof for Proposition \ref{prop:godel-impl}. 
\begin{proof}
    First, assume $\revis{R_G} <  1$. To ensure $R_G(\truths_1, \truths_2)=\revis{R_G}$, we require $\truths_2 =\revis{R_G}$ as is clear from the definition. However, we also require $\truths_1 > \revis{R_G}$. If $\truths_1$ is already larger, we can leave it to ensure minimality. Otherwise, we require it to be at least infinitesimally bigger, that is $\revis{R_G} + \epsilon$. 

    Next, assume $\revis{R_G}=1$. If $\truths_1\leq \truths_2$, then the implication is already 1 and we do not need to revise anything. Otherwise, setting it equal to any value between $\truths_2$ and $\truths_1$ is minimal. \emile{should probably formalize this}
\end{proof}

\subsection{\luk\ t-norm minimal \boosted\ function proofs}
\subsubsection{\luk\ t-norm}
\label{appendix:luk-tnorm}
\begin{proposition}
    \label{prop:luk-t-norm}
    % Let $I\subseteq D$, $\weight\in[T_L(\truth), \weightmax]$ and define $\minboosts_I=\frac{w+\vert F\vert + \vert I \vert -1-\sum_{i\in I\cup F}t_i}{\vert I\vert}$. Let $\pi(i)$ denote the permutation of $\truths_i$ in increasing order and let $K=\max_i \minboosts_{\{\pi^ {-1}(1), ..., \pi^ {-1}(n)\}}<1-\truth_{\pi(i)}$ and $I^ *=\{i\in \{0, ..., n\} \vert \pi^ {-1}(i) \}\leq K$. Then the minimal t-norm boost vector of the \luk\ t-norm is 
     Let $\revis{T_L}\in[T_L(\truth, \truthc), \max(\|\truthc\|_1 - (m - 1), 0)]$ and define $\lukincrease_K=\frac{\revis{T_L}+ m + K -1-\|\truthc\|_1 - \sum_{i=1}^ K\truths^\uparrow_{i}}{K}$. Let $K^ *$ be the largest integer $1\leq K\leq \vert D \vert$ such that $\lukincrease_{K}<1-\truths^ \uparrow_{K}$. Then the minimal \boost\ vector of the \luk\ t-norm is 
\begin{equation}
    \minboost_{T_L}(\truth, \revis{T_L})_i=\begin{cases}
        \truths_i + \lukincrease_{K^*} & \text{if } \truths_i \leq \truths^\uparrow_{K}, \\
        1 & \text {otherwise}
    \end{cases}
\end{equation}
\end{proposition}
\begin{proof}
    We will prove this using the KKT conditions, which are both necessary and sufficient for minimality for the \luk\ t-norm since it is affine when the max constraint is not active. We drop the $p$-root in the norm since it is a strictly monotonically increasing function.  The Lagrangian and corresponding derivative is
    \begin{align*}
        \ell&=\sum_{i=1}^n \vert\boosts_i -\truths_i\vert^p + \lambda(\max(\|\boostv\|_1 + \|\truthc\|_1 - (m+n-1), 0) - \revis{T_L}) +\sum_{i=1}^ n \gamma_i(\boosts_i -1)\\
        \frac{\partial \ell}{\partial  \boosts_i} &= p(\boosts_i - \truths_i)^{p-1} + \lambda \frac{\partial }{\partial \boosts_i} \max(\|\boostv\|_1 + \|\truthc\|_1 - (m+n-1), 0) + \gamma_i =0.
    \end{align*}
    We note that we drop the absolute signs since $T_L$ is strictly monotonically increasing function and $\revis{T_L} \geq T_L(\truth, \truthc)$. 
    Assuming $\revis{T_L} > 0$, $T_L(\boostv, \truthc)=\revis{T_L}$ can only be true if the first argument of $\max$ is chosen. Then for all $i, j\in \{1, ..., n\}$, $p(\boosts_i - \truths_i)^ {p-1} + \gamma_i=p(\boosts_j - \truths_j)^ {p-1} +  \gamma_j$. Define $I$ as the set of $K^ *$ smallest $\truths_i$. 
    \begin{itemize}
        \item \emph{Primal feasibility:} For all $i\in I$, $\minboost_{T_L}(\truth, \revis{T_L})_i=\lukincrease_{K^ *}\leq 1$ by definition. For all $i\in \{1, ..., n\}\setminus I$, $\minboost_{T_L}(\truth, \revis{T_L})_i=1-\truths_i$. Furthermore, 
        \begin{align*}
        T_L(\minboost_{T_L}(\truth, \revis{T_L}), \truthc) &=\max(\sum_{i=1}^{K^*} (\truths^\uparrow_i + \lukincrease_{K^ *}) +\sum_{i=K^* +1}^n 1 + \|\truthc\|_1 - n - m + 1, 0) \\
        &= \max(\sum_{i=1}^{K^*} \truths^\uparrow_i + K^ * \lukincrease_{K^ *} + n - K^ * + \|\truthc\|_1 - n - m + 1, 0) \\
        &= \max(\sum_{i=1}^{K^*}\truths_i^\uparrow + \revis{T_L} + m + K^ * -1 -\|\truthc\|_1 - \sum_{i=1}^{K^*} \truths^\uparrow_i  \\
        &- K^ *   + \|\truthc\|_1  - m+ 1, 0)  =\revis{T_L}
        \end{align*}
        \item \emph{Complementary Slackness:} Clearly, for all $i\in I$, we require $\gamma_i=0$. For all $i\in \{1, ..., n\}\setminus I$, $\minboost_{T_L}(\truth, \revis{T_L}t)_i -1 = 1 - 1 = 0$. 
        \item \emph{Dual feasibility:} 
        % \todo{We use the uparrow notation here which is only introduced in the Schur concavity section. Need to move that forward if this proof makes it to the main paper} 
        For all $i\in I$, $\gamma_i=0$. For $i\in \{1, ..., n\}\setminus I$, consider some $j\in I$ and note that $p(\boosts_i - \truths_i)^ {p-1} + \gamma_i=p(\boosts_j - \truths_j)^ {p-1} +  \gamma_j$. Filling in $\boostv$, we find $\gamma_i= p\lukincrease_{K^ *}^ {p-1} - p(1-\truths_i)^ {p-1}$. This is nonnegative if $\lukincrease_{K^ *}\geq 1-\truths_i$. First, we show $\lukincrease_{K^*} \geq \lukincrease_{K^*+1}$. Write out their definitions, multiply by $K^*(K^*+1)$ and remove common terms. Then,
        \begin{align*}
            \revis{T_L} + m - 1 - \|\truthc\|_1 - \sum_{i=1}^{K^*}\truths_{i}^\uparrow &\geq -K^* \truths^\uparrow_{K^*+1}\\
            \revis{T_L} + m + K^*+1 -\|\truthc\|_1 - \sum_{i=1}^{K^*+1}\truths_{i}^\uparrow &\geq (K^*+1)(1-\truths_{K^*+1}^\uparrow) \\
            \lukincrease_{K^*+1}&\geq 1-\truths_{K^*+1}^\uparrow.
        \end{align*}
        $\lukincrease_{K^*+1}\geq 1-\truths_{K^*+1}^\uparrow$ is true by the construction in the proposition. Therefore, 
        \begin{equation*}
            \lukincrease_{K^*}\geq\lukincrease_{K^*+1}\geq 1-\truths^\uparrow_{K^*+1}\geq 1-\truths_i,
        \end{equation*}
        proving dual feasibility.
    \end{itemize}
\end{proof}   
\subsubsection{\luk\ t-conorm}
\label{appendix:luk-tconorm}
\begin{proposition}
    \label{prop:luk-t-conorm}
    The minimal \boost\ function of the \luk\ t-conorm for $\revis{S_L}\in[S_L(\truth, \truthc), \revis{S_L}]$ is
\begin{equation}
    \minboost_{S_L}(\truth, \revis{S_L})_i=
        \truths_i + \frac{\max(\revis{S_L}-\|\truth\|_1 - \|\truthc\|_1, 0)}{n}
\end{equation}
\end{proposition}
\begin{proof}
    We do not add multipliers for the constraints on $\boosts_i$, and show critical points adhere to these constraints. The Lagrangian is
\begin{equation}
    \ell=\sum_{i=1}^n (\boosts_i - \truths_i)^p + \lambda(\min(\|\boostv\|_1 + \|\truthc\|_1, 1) - \revis{S_L})
\end{equation}
Note that $\revismax{S_L}=1$. 
Taking the derivative to $\boosts_i$, we find
\begin{align*}
    \frac{\partial \ell}{\partial \boosts_i} =  p\cdot (\boosts_i - \truths_i)^{p-1} + \lambda\frac{\partial}{\partial \boosts_i} \min(\min(\|\boostv\|_1 + \|\truthc\|_1, 1)=0
\end{align*}
Assume $\revis{S_L}\neq S_L(\truth)$, this gives three cases for all $i\in \{1, ..., n\}$:
\begin{enumerate}
    \item If $\|\truth\|_1 + \|\truthc\| \geq 1$ and $\revis{S_L}=1$, then since $\boosts_i \geq \truths_i$, $\frac{\partial}{\partial \boosts_i} \min(\|\boostv\|_1 + \|\truthc\|_1, 1)=\frac{\partial}{\partial \boosts_i}1=0$, and so $\boosts_i=\truths_i$.
    \item If $\|\truth\|_1 + \|\truthc\| \geq 1$, then $\revismin{S_L} =\revismax{S_L}= 1$, and again $\boosts_i=\truths_i$. 
    \item Otherwise, it must be that $\|\boostv\|_1 + \|\truthc\|_1\leq 1$ and so $\frac{\partial}{\partial \boosts_i} \min(\|\boostv\|_1 + \|\truthc\|_1, 1)=\frac{\partial}{\partial \boosts_i}\|\boostv\|_1=1$, and therefore $p\cdot(\boosts_i - \truths_i)^{p-1}=-\lambda$. \todo{Note that if $p=1$, this means the gradient is constant everywhere. This means the proof doesn't work for p=1}  Since the equality holds for all $i\in \{1, ..., n\}$, we find $p\cdot(\boosts_i - \truths_i)^{p-1}=p\cdot(\boosts_j - \truths_j)^{p-1}$ for all $i, j\in \{1, ..., n\}$. As we are only interested in real nonnegative solutions,  we find that $\boosts_i - \truths_i=\boosts_j - \truths_j=\delta$. 
    Since $\|\boostv\|_1 + \|\truthc\|_1  = \|\truth\|_1 + \|\truthc\|_1 + n \delta  = \revis{S_L}$, we find
    \begin{equation*}
        % \sum_{k\in D} \truths_k + \delta = \revis{S_L}-\sum_{i\in F} \truths_k, 
        % \quad
        % n\boosts_i &=\weight-\sum_k \truths_k \\
        \delta = \frac{\revis{S_L}-\|\truth\|_1 - \|\truthc\|_1}{n}, \quad \boosts_i=\truths_i + \delta.
    \end{equation*}
    % Finally, we need to show that for all $i$, $-\truths_i\leq \boosts_i \leq 1-\truths_i$. 
    Note that $\boosts_i\geq \truths_i$, since by assumption $\revis{S_L}\geq S_L(\truth, \truthc)$, and $\boosts_i \leq 1$ since by $\revis{S_L} \leq \revismax{S_L}\leq1$, $\delta = \frac{\revis{S_L}-\|\truth\|_1 - \|\truthc\|_1}{n} \leq \frac{1-\|\truth\|_1 - \|\truthc\|_1}{n}\leq \frac{1-\truths_i}{n}\leq 1-\truths_i$, that is, the constraints of Equation \ref{eq:optim-problem} are satisfied.

    % $\boostv _i=\frac{w-\sum_k \truths_k}{\vert D \vert} \geq \frac{-\sum_k t_k}{\vert D \vert}\leq \frac{1-\truths_i}{\vert D \vert}\leq 1-\truths_i$. \todo{I don't think this works out? Need to constrain it better.}
    % \todo{But wait. Weren't we going to do it by just proving w> f(t)? And then using proposition 1 to do the negativity... yes i think so. let's keep it like that. }
\end{enumerate}
\end{proof}%

\section{Dual Problem}
This section introduces a dual problem to Equation \ref{eq:optim-problem}. This is used extensively in several proofs. 
\subsection{Strict cone monotonicity}
\label{appendix:cone-monotonicity}
\begin{definition}
    \label{def:cone-monotone}
    A set $K\subset [0, 1]^n$ is a \emph{(convex) cone} if for every $s>0$ and $\truth\in K$ such that $s\truth\in [0, 1]^n$, also $s\truth \in K$. 

    A fuzzy evaluation operator $\op_\varphi$ is \emph{strictly cone-increasing} at $\truth\in [0, 1]^n$ if there is a nonempty cone $K(\truth)$ such that $\truth'-\truth \in K$ implies $\op_\varphi(\truth) < \op_\varphi(\truth')$.
\end{definition}

Strict cone-monotonicity is a weak notion of monotonicity in the sense that all t-norms that are strictly increasing in each argument are strictly cone-increasing, but the reverse need not be true. 

\begin{proposition}
    \label{prop:cone-monotone}
    If $f_\varphi$ is non-decreasing and strictly cone-increasing at $\truth \in [0, 1]^n$, there exist a nonempty cone $K'(\truth)\subseteq K(\truth)$ such that $\truth'-\truth\in K'(\truth)$ implies $\truth'_i \geq \truths_i$ for all $i\in \{0, ..., n\}$.
\end{proposition}

\begin{proof}
    Assume otherwise. Consider some $\truth'$ such that $s(\truth'-\truth)\in K(\truth)$ for $s>0$. By assumption, there is some $i\in \{0, ..., n\}$ such that $\truth'_i<\truths_i$. Consider $\hat{\truth}$ equal to $\truth'$ except that $\hat{\truth}_i=\truths_i$ for such $i$. Since $f_\varphi$ is non-decreasing in each argument, $f_\varphi(\hat{\truth})\geq f_\varphi(\truth') > f_\varphi(\truth)$, then clearly $s(\hat{\truth}-\truth)$ for $s>0$ forms the cone $K'(\truth)$.
\end{proof}
\subsection{Dual problem}
Next, we will investigate a dual problem  for the problem in Equation \ref{eq:optim-problem} that will allow us to prove multiple useful theorems:
\begin{equation}
    \label{eq:optim-problem-max}
    \begin{aligned}
        \textrm{For all } \quad & \truth\in [0,1]^n, u \in [0, \infty):  \\
        \max_{\boostv} \quad & \op_\varphi(\boostv)   \\
        \textrm{such that } \quad & \|\boostv - \truth\| = u,  \\
        & 0\leq  \boosts_i \leq 1. 
        % & \boosts_i=0, i\in F 
    \end{aligned}
\end{equation}
% \looseness=-1
That is, instead of finding the $\boostv$ closest to $\truth$ with \boost\ value $\revis{\varphi}$, we find the largest \boosted\ value attainable with a fixed budget $u$. We need to be precise when solutions of this dual problem coincide with the problem in Equation \ref{eq:optim-problem}. We consider strict cone-monotonicity \cite{vandykeConeMonotonicityStructure2013,clarkeSubgradientCriteriaMonotonicity1993}, which is a weak notion of strict monotonicity for higher dimensions. This intuitively means that there is always some direction we can move in to increase the value of the t-norm. Since t-norms are already non-decreasing in each argument, this implies there is no point where the t-norm is \say{flat} in all directions. The precise definition is given in Definition \ref{def:cone-monotone}.

\begin{theorem}
    \label{theorem:max-to-min}
   A solution $\minboostv$ for some $\op_\varphi$, $\truth$ and $u\geq 0$ of Equation \ref{eq:optim-problem-max} is also a solution to Equation $\ref{eq:optim-problem}$ for $\truth$ and $\revis{\varphi}=\op_{\varphi}(\minboostv)\geq\op_{\varphi}(\truth)$ if $\op_\varphi$ is non-decreasing in all arguments and strictly cone-increasing at each $\truth'\in [0, 1]^n$ such that $\op_\varphi(\truth')=\revis{\varphi}$, and if $\|\cdot \|$ is strictly increasing in all arguments.
\end{theorem}
\begin{proof}
    Assume otherwise, and suppose a solution $\boostv$ for Equation \ref{eq:optim-problem} exists such that $\op_\varphi(\boostv)=\revis{\varphi}$ while $\| \boostv - \truth\| < \| \minboostv  - \truth\|=u$. Since $\op_\varphi$ is non-decreasing in all arguments and $\revis{\varphi} \geq \op_\varphi(\truth)$, $\boostv - \truth$ and $\minboostv  - \truth$ are nonnegative. By Proposition \ref{prop:cone-monotone} there is some cone $K(\boostv)$ that contains a line segment $\boldsymbol{\epsilon}(s)=s(\truth'-\boostv)$ such that for all $s>0$, $\op_\varphi(\boostv) < \op_\varphi(\boostv + \mathbf{\epsilon}(s))$ and for all $i\in \{0, ..., n \}$, $0 \leq \boldsymbol{\epsilon}(s)_i$. Therefore, necessarily there is some $i$ such that $0 < \boldsymbol{\epsilon}(s)_i$. Since $\|\cdot \|$ is strictly increasing on nonnegative vectors and continuous (since it is a norm), necessarily there is some $s>0$ such that $\|\boostv+\boldsymbol{\epsilon}(s)\|=u$. However, this is in contradiction with the premise that $\minboostv$ is a solution of Equation \ref{eq:optim-problem-max}, as $\op_\varphi(\boostv+\epsilon(s)) > \op_\varphi(\minboostv)$. 
\end{proof}

Since $f_\varphi\in[0, 1]^n \rightarrow [0, 1]$, $\op_\varphi$ cannot satisfy the conditions of Theorem \ref{theorem:max-to-min} when $\revis{\varphi}=1$. For all $\revis{\varphi}\in[0, 1)$ however, both the G\"{o}del and product t-norms and t-conorms are strictly cone-increasing. The \luk{} t-norm satisfies the conditions for $\revis{\varphi}\in (0, 1)$, since it has flat regions for $\revis{\varphi}=0$. The same reasoning can be made for the nilpotent minimum and drastic t-norms, see \cite{vankriekenAnalyzingDifferentiableFuzzy2022}. Furthermore, all t-norms with an additive generator are strictly cone-increasing on $\revis{\varphi} \in (0, 1)$, as are all strict t-norms.

\section{Schur-concave t-norms (Proofs)}
\label{appendix:schur-concave}

\subsection{Minimal \boost\ function for t-norms}
\label{appendix:boost-schur-tnorm}

\begin{theorem}
    \label{theorem:schur-concave-t-norm}
    Let $T$ be a Schur-concave t-norm that is strictly cone-increasing at $\revis{T}$ and let $\|\cdot \|$ be a strict norm. Then there is a minimal \boosted\ vector $\minboostv$ for $\truth$ and $\revis{T}$ such that whenever $\truths_i> \truths_j$, then $\minboosts_i - \truths_i\leq \minboosts_j - \truths_j$. 
\end{theorem}
\begin{proof}
    Assume there is a minimal \boosted\ vector $\boostv\neq\minboostv$ which has some $\boosts_i - \truths_i > \boosts_j - \truths_j$ while $\truths_i > \truths_j$. Consider $\boostv'$ equal to $\boostv$ except that $\boosts'_i=\boosts_j - \truths_j + \truths_i$ and $\boosts'_j=\boosts_i - \truths_i + \truths_j$ such that by symmetry $\|\boostv - \truth\|=\|\boostv' - \truth\|$. Define $\newmax'=\max( \boosts_i', \boosts_j')$ and $\newmin' = \min(\boosts_i', \boosts_j')$. Clearly, $\boosts_i> \newmax'\geq\newmin'> \boosts_j$. 
    % Let $\bx = (\truth + \boostv)^\downarrow$ and $\bx' = (\truth + \boostv')^ \downarrow$, i.e., the new truth values sorted in descending order. 
    We will show $\boostv$ majorizes $\boostv'$ by checking the condition of Definition \ref{def:schur-concave} for any $k\in \{1, ..., n\}$. 
    \begin{enumerate}
        \item If $\boosts^\downarrow_k> \boosts_i$, then all elements are equal and $\sum_{l=1}^ k \boosts^\downarrow_l=\sum_{l=1}^ k \boosts'^\downarrow _l$. 
        \item If $\boosts_i \geq \boosts^\downarrow_k > \newmax'$, then $\sum_{l=1}^ {k}\boosts^\downarrow_l=\sum_{l=1}^ {k-1} \boosts'^\downarrow_l + \boosts_i \geq \sum_{l=1}^ k \boosts'^\downarrow_l$.
        \item If $\newmax' \geq \boosts^\downarrow_k > \newmin'$, then $\sum_{l=1}^k \boosts^\downarrow_l > \sum_{l=1}^k \boosts'^\downarrow_l$, since by removing common terms we get $\boosts_i > \newmax'$. 
        \item If $\newmin'\geq \boosts^\downarrow_k > \boosts_j$, then removing all common terms in the sums, we are left with $\boosts_i + \boosts^\downarrow_k > \newmin'+\newmax'$. Note $\newmin' + \newmax' = \boosts_j + \truths_i - \truths_j  + \boosts_i + \truths_j-\truths_i=\boosts_i + \boosts_j$. Subtracting $\boosts_i$ from both sides, we are left with $\boosts^\downarrow_k > \boosts_j$, which is true by assumption.
        \item If $\newmin \geq \boosts^\downarrow_k$, then removing common terms, we are left with $\newmax + \newmin=\boosts_i + \boosts_j$.
    \end{enumerate}
    Therefore, $\boostv$ majorizes $\boostv'$, and so by Schur concavity, $T(\boostv, \truthc)\leq T(\boostv', \truthc)$, noting that the additional truth vector $\truthc$ will not influence the majorization result since it is applied at both sides. By Theorem \ref{theorem:max-to-min}, either 1) $T(\boostv, \truthc)< T(\boostv', \truthc)$, so $\boostv$ could not have been minimal, leading to a contradiction, or 2) $T(\boostv, \truthc)=T(\boostv', \truthc)$ and both $\boostv$ and $\boostv'$ are minimal.
\end{proof}
\subsection{Closed-form \boost\ function using additive generators}
\label{appendix:closed-form-additive}
\begin{proposition}
    \label{prop:additive-generator}
    Let $T$ be a Schur-concave t-norm with additive generator $g$ and let $0<\revis{T}\in [T(\truth, \truthc), \revismax{T}]$. 
    Let $K\in \{0, ..., n-1\}$ denote the number of truth values such that $\minboosts_i=\truths_i$ in Equation \ref{eq:minboost-t-norm-schur-concave}.
    % Let $N\subseteq \{1, ..., n\}$ be the set of indices $i$ such that the $\minboostv$ from Equation \ref{eq:minboost-t-norm-schur-concave} has $\minboosts_i=\truths_i$. 
    Then using

    \begin{equation}
        \lambda_K = g^ {-1}\left(\frac{1}{n-K}\left(g(\revis{T}) -\sum_{i=1}^K g(\truths^\downarrow_i) - \sum_{i=1}^m g(C_i)\right)\right)
    \end{equation}

    in Equation $\ref{eq:minboost-t-norm-schur-concave}$ gives $T(\minboostv, \truthc)=\revis{T}$ if $\minboostv\in [0, 1]^n$. 
\end{proposition}
\begin{proof}
    Using Equations \ref{eq:additive-generator} and \ref{eq:minboost-t-norm-schur-concave}, we find that 
    \begin{align*}
        T(\minboostv, \truthc)=g^{-1}\left(\min\left(g(0^+), \sum_{i=1}^K g(\truths^\downarrow_i) + \sum_{i=K+1}^n g(\lambda_K) + \sum_{i=1}^m g(C_i)\right)\right)=\revis{T}
    \end{align*}
    Since $\revis{T} > 0$, we can remove the $\min$, since $\revis{T}>0$ will require that $\sum_{i=1}^K g(\truths^\downarrow_i) + (n-K) g(\lambda_K) + \sum_{i=1}^m g(C_i)>g(0^+)$. We apply $g$ to both sides of the equation, which is allowed since $g$ is a bijection. Thus 
    \begin{align*}
        g(\revis{T}) &= \sum_{i=1}^K g(\truths^\downarrow_i) + (n-K) g(\lambda_K) + \sum_{i=1}^m g(C_i)  \\
        g(\lambda_K) &= \frac{1}{n-K}\left( g(\revis{T}) - \sum_{i=1}^K g(\truths^\downarrow_i) - \sum_{i=1}^m g(C_i) \right) \\
        \lambda_K &= g^ {-1}\left(\frac{1}{n-K}\left(g(\revis{T}) -\sum_{i=1}^K g(\truths^\downarrow_i) - \sum_{i=1}^m g(C_i)\right)\right),
    \end{align*}
    where in the last step we apply $g^{-1}$.
\end{proof}
In a similar manner we can find the $\lambda$ for the t-conorm. Let $j={\arg\max}_{i=1}^n \truths_i$. 
\begin{align*}
    S(\minboostv) = 1-g^{-1}(\min (g(0^+), \sum_{i=1}^n g(1-\minboosts_i) + \sum_{i=1}^m g(1-C_i))) &= \revis{S} \\
    \min(g(0^ +), g(1-\lambda)+ \sum_{i\neq j} g(1-\truths_i) + \sum_{i=1}^m g(1-C_i)) &= g(1-\revis{S})
\end{align*}
If $\revis{S} < 1$, or if $g(0)$ is well defined, then we can ignore the $\min$:
\begin{align*}
    g(1-\lambda) &= g(1-\revis{S}) - \sum_{i\neq j} g(1-\truths_i)) -  \sum_{i=1}^m g(1-C_i) \\
    \lambda &= 1 - g^{-1}\left(g(1-\revis{S}) - \sum_{i\neq j}g(1-\truths_i) -  \sum_{i=1}^m g(1-C_i)\right) 
\end{align*}

\subsection{L1 minimal \boost\ function for t-norms}
\label{appendix:boost-schur-t-norm-L1}
\emile{The $\lambda$ to choose is the one that satisfies those conditions and reaches $\weight$. But it turns out in the experiments that this lambda is also the smallest lambda, that is adding more or less elements to $N$ increases the value of $\lambda$? This is useful because it makes the implementation a lot easier... Is this something we can prove?}
\begin{proposition}
    Let $\truth\in [0, 1]^n$ and let $T$ be a Schur-concave t-norm that is strictly cone-increasing at $\revis{T}\in [T(\truth, \truthc), \revismax{T}]$. Then there is a value $\lambda\in [0, 1]$ such that the vector $\minboostv$,
    \begin{equation}
        \label{eq:minboost-t-norm-schur-concave}
        \minboosts_i = \begin{cases}
            \lambda, & \text{if } \truths_i < \lambda, \\
            \truths_i, & \text{otherwise,}
        \end{cases}
    \end{equation}
    is a minimal \boosted\ vector for $T$ and the L1 norm at $\truth$ and $\revis{T}$.
\end{proposition}
\begin{proof}
    Assume otherwise. Then, using Theorem \ref{theorem:max-to-min}, there must be a \boosted\ vector $\boostv$ such that $\|\boostv - \truth\|_1=\|\minboostv - \truth\|_1$ but $T( \boostv, \truthc) > T(\minboostv, \truthc)$. Since $\revis{T} \in [T(\truth, \truthc), \revismax{T}]$, we can assume $\boosts_i\geq \truths_i$. 
    % We define $\bx^*=(\truth+\minboostv)^ \downarrow$ and $\bx=(\truth + \boostv)^ \downarrow$. 
    We define $\pi^*(i)$ as the permutation in descending order of $\minboostv$. Furthermore, let $k$ be the smallest $j$ such that $\truths^ \downarrow_j < \lambda$. 

    Since $\|\boostv\|_1=\|\minboostv\|_1$, by assumption of equal L1 norms of $\boostv$ and $\minboostv$, we will prove for all $i\in \{1, ..., n\}$ that $\boostv$ majorizes $\minboostv$. 
    \begin{itemize}
        \item If $i < k$, then $\sum_{j=1}^i\boosts^\downarrow_j\geq\sum_{j=1}^i \boosts_{\pi^ *(j)} \geq \sum_{j=1}^i \truths_{\pi^*(j)}=\sum_{j=1}^i\minboosts^\downarrow_j$. The first inequality follows from the fact that there is no ordering of $\boostv$ that will have a higher sum than in descending order. 
        \item If $i\geq k$, then clearly $\minboosts^\downarrow_i=\lambda$. Furthermore, $\sum_{j=1}^i \minboosts^\downarrow_j=\sum_{j=1}^k\truths^\downarrow_j+(i-k)\lambda$. We will distinguish two cases:
        \begin{enumerate}
            \item $\boosts^\downarrow_i \geq \lambda$. Then for all $j\in \{k, ..., i\}$, $\boosts^\downarrow_j\geq \lambda$. Furthermore, from the previous result, $\sum_{j=1}^{k-1}\boosts^\downarrow_j \geq \sum_{j=1}^ {k-1}\minboosts^\downarrow_j$ and so clearly $\sum_{j=1}^i \boosts^\downarrow_j \geq \sum_{j=1}^ i\minboosts^\downarrow_i$.
            \item $\boosts^\downarrow_i < \lambda$. Then for all $j>i$, $\boosts^\downarrow_j\leq \boosts^\downarrow_i< \lambda$, and so $\sum_{j=i+1}^n \boosts^\downarrow_j \leq \sum_{j=i+1}^ n \boosts^\downarrow_i=(n-i)\boosts^\downarrow_i<(n-i)\lambda$. Using this, we note that 
            \begin{equation*}
                \|\minboostv\|_1=\sum_{j=1}^k\truths^\downarrow_j+(n-k)\lambda=\|\boostv\|_1=\sum_{j=1}^i\boosts^\downarrow_j+\sum_{j=i+1}^n\boosts^\downarrow_j<\sum_{j=1}^i \boosts^\downarrow_j + (n-i)\lambda.
            \end{equation*} 
            Then, subtracting $(n-i)\lambda$ from the inequality, we find
            \begin{align*}
                \sum_{j=1}^i\boosts^\downarrow_j> \sum_{j=1}^k\truths^\downarrow_j+(n-k)\lambda-(n-i)\lambda=\sum_{j=1}^k\truths^\downarrow_j+(i-k)\lambda=\sum_{j=1}^i \minboosts^\downarrow_j
            \end{align*}
        \end{enumerate}
    \end{itemize}
    And so, $\boostv$ majorizes $\minboostv$, and by Schur concavity of $T$, $T(\boostv, \truthc) \leq T(\minboostv, \truthc)$ leading to a contradiction. 
\end{proof}

\subsection{L1 minimal \boost\ function for t-conorms}
\label{appendix:schur-concave-t-conorm-l1}
\begin{proposition}
    Let $\truth\in [0, 1]^n$ and let $S$ be a Schur-convex t-conorm that is strictly cone-increasing at $\revis{S}\in [S(\truth, \truthc), 1]$. Then there is a value $\lambda\in [0, 1]$ such that the vector $\minboostv$,
    \begin{equation}
        \minboosts_i = \begin{cases}
            \lambda & \text{if } i={\arg\max}_{i\in D}\truths_i, \\
            \truths_i, & \text{otherwise,}
        \end{cases}
    \end{equation}
    is a minimal \boosted\ vector for $S$ and the L1 norm at $\truth$ and $\revis{S}$.
\end{proposition}
\begin{proposition}
    Let $\truth\in [0, 1]^n$ and let $S$ be a Schur-convex t-conorm that is strictly cone-increasing at $\revis{S}\in [S(\truth, \truthc), 1]$. Then there is a value $\lambda\in [0, 1]$ such that the vector $\minboostv$ with $i\in D$,
    \begin{equation}
        \minboosts_i = \begin{cases}
            \lambda & \text{if } i={\arg\max}_{i\in D}\truths_i, \\
            \truths_i, & \text{otherwise,}
        \end{cases}
    \end{equation}
    is a minimal \boosted\ vector for $S$ and the L1 norm at $\truth$ and $\revis{S}$.
\end{proposition}
\begin{proof}
   Assume otherwise. Then, using Theorem \ref{theorem:max-to-min}, there must be a \boosted\ vector $\boostv\neq \minboostv$ such that $\|\boostv - \truth\|_1=\|\minboostv - \truth\|_1=\lambda-\truth^\downarrow_1$ but $S(\boostv, \truthc) > S(\minboostv, \truthc)$. 
%    Define $\bx=(\truth + \boostv)^\downarrow$ and $\bx^*=(\truth+\minboostv)^\downarrow$, 
   Let $\pi(i)$ be the permutation in descending order of $\boostv$. 
%    Let $j$ be such that $\minboostv^\downarrow_j=\truth^\downarrow_j+\lambda$ and note that $\bx^*_{i}=\truth^\downarrow_i$ for $i \neq j$. 

%    First, consider $k < j$. Note that each $i\in \{1, ..., j-1\}$ has that $i\in F$ by construction. Furthermore, $\boosts_i\leq \lambda$ for $i\in D$, and so $\bx_{j-1}\leq \truth^\downarrow_{j-1}$. Therefore, $\sum_{i=1}^k\minboostv^\downarrow_i=\sum_{i=1}^k\truth^\downarrow_i=\sum_{i=1}^k \boostv^\downarrow_i$

%    Secondly, consider $k\geq j$. 
   Consider any $k\in \{1, ..., n\}$. Then $\sum_{i=1}^k \minboosts^\downarrow_i=\sum_{i=1}^k\truths^\downarrow_i+(\lambda - \truths^\downarrow_1)$, while $\sum_{i=1}^k \boosts^\downarrow_j=\sum_{i=1}^k\truths_{\pi(i)}+\sum_{i=1}^k(\boosts_{i} - \truths_{\pi(i)})$. There is no permutation with higher sum than in descending order, so $\sum_{i=1}^k\truths_{\pi(i)} \leq \sum_{i=1}^k\truths^\downarrow_i$. Furthermore, since $\|\boostv - \truth\|_1=\lambda - \truths^\downarrow_1$, $\sum_{i=1}^k(\boosts_i - \truths_{\pi(i)})\leq \lambda - \truths^\downarrow_1$. Therefore, $\sum_{i=1}^k \boosts^\downarrow_i \leq \sum_{i=1}^k \minboosts^\downarrow_i$, that is, $\minboostv$ majorizes $\boostv$, and by Schur convexity of $S$, $S(\minboostv, \truthc) \geq S(\boostv, \truthc)$. 
\end{proof}

\subsection{L1 minimal \boost\ function for residuums}
\label{appendix:boost-schur-residuum}
\todo{This doesn't consider constants}
    \begin{proposition}
        Let $\truths_1, \truths_2\in [0,1]$ and let $T$ be a strict Schur-concave t-norm with additive generator $g$. Consider its residuum $R(\truths_1, \truths_2)=\sup \{z\vert T(\truths_1, z)\leq \truths_2\}$ that is strictly cone-increasing at $0<\revis{R}\in [R(\truths_1, \truths_2), \revismax{R}]$. Then there is a value $\lambda\in [0, 1]$ such that $\minboostv=[\truths_1, \lambda]^\top$ is a minimal \boosted\ vector for $R$ and the L1 norm at $\truth$ and $\weight$. 
    \end{proposition}
    
\begin{proof}
    We will assume $\truths_1 > \truths_2$, as otherwise $R(\truths_1, \truths_2)=1$ for any residuum, which necessarily means $\revis{R}=1$ and so $\minboostv=\truth$. 
    Assume $\minboostv$ is not minimal. Since $R$ is strictly cone increasing at $\revis{R}$, by Theorem \ref{theorem:max-to-min}\footnote{This theorem has to be adjusted for the fact that fuzzy implications are non-increasing in the first argument. It can be applied by considering $1-\truths_1$.} there must be some $\boostv$ such that $\|\boostv-\truth\| = \|\minboostv - \truth\|=\lambda-\truths_2$ but $R(\boosts_1, \boosts_2)>R(\minboosts_1, \minboosts_2)$. Since $R$ is non-decreasing in the first argument and non-increasing in the second, we consider $\boostv=[\truths_1 - \epsilon, \lambda - \epsilon]^\top$ for $\epsilon>0$. 
    
    The residuum constructed from continuous t-norms with an additive generator can be computed as $R(\truths_1, \truths_2)=\add^{-1}(\max(\add(\truths_2)-\add(\truths_1), 0))$. Since we assumed $R(\minboosts_1, \minboosts_2) < R(\boosts_1, \boosts_2)$, applying $g$ to both sides,
    \begin{align*}
       \max(g(\lambda) - g(\truths_1), 0) &> \max(g(\lambda - \epsilon) - g(\truths_1 - \epsilon), 0)\\
       g^{-1}(g(\lambda) + g(\truths_1 - \epsilon)) &< g^{-1}(g(\lambda-\epsilon) + g(\truths_1))\\
       T(\lambda, \truths_1 - \epsilon) & < T(\lambda - \epsilon, \truths_1)
    \end{align*}
    where in the second step we assume $\lambda \leq \truths_1$, that is, we are not setting new consequent larger than the antecedent, as otherwise we could find a smaller \boosted\ vector by setting it to exactly $\truths_1$. In the last step we use that $T$ is strict, as then $T(\truths_1, \truths_2)=g^{-1}(g(\truths_1) + \truths_2))$. We now use the majorization as $\lambda + \truths_1 - \epsilon = \lambda - \epsilon + \truth$.

    Since $\lambda \leq \truths_1$, surely $\truths_1 > \lambda - \epsilon$. Then there are two cases: 
    \begin{enumerate}
    \item $\lambda \geq \truths_1 -\epsilon$. Then $\truths_1 \geq \lambda$ as assumed.
    \item $\truths_1 - \epsilon \geq \lambda$. Then clearly $\truth \geq \truths_1 - \epsilon$ as $\epsilon > 0$. 
    \end{enumerate}
    Therefore $[\lambda - \epsilon, \truths_1]^\top$ majorizes $[\lambda, \truth- \epsilon]^\top$, and by Schur concavity $T(\lambda, \truth - \epsilon) \geq T(\lambda - \epsilon, \truths_1)$ which is a contradiction.
\end{proof}

\section{Product t-norm with L2 norm}
\label{Appendix:product-l2}
In this appendix, we consider the \boost\ functions for the product t-norm under the L2-norm. We find that there is no simple closed-form parameterization in terms of $\revis{\varphi}$, but we can find approximations in linear time. These are satisfactory to reliably find the minimal \boost\ function.

In the following, we wil ignore constants and consider formulas $\bigwedge_{i=1}^n P_i$, and consider the problem in Equation \ref{eq:optim-problem}. We consider the logarithm of the product as its optimum coincides. 

\begin{equation}
    \label{eq:optim-problem-prod-p1}
    \begin{aligned}
        \textrm{For all } \quad & \truth\in [0,1]^n, \revis{T_P} \in [T_P(\truth), 1] & \\
        \min_{\boostv} \quad & \sum_{i=1}^n (\boosts_i - \truths_i)^2 & \\
        \textrm{such that } \quad & \sum_{i=1}^n \log \boosts_i = \log \revis{T_P} \\
        & \boosts_i-1 \leq 0 
    \end{aligned}
\end{equation}
With Lagrangian $L=\sum_i (\boosts_i - \truths_i)^2 + \lambda(\sum_i \log \boosts_i -\log \revis{T_P})  - \gamma_i(\boosts_i - 1)$, and so
\begin{align*}
    \frac{\partial L}{\partial \boosts_i} &= 2(\boosts_i - \truths_i) - \gamma_i +\frac{\lambda}{\boosts_i}=0 \\
    \lambda &= (\gamma + 2\truths_i - 2\boosts_i) \boosts_i %(\gamma_i-2\boosts_i)(\truths_i + \boosts_i)
\end{align*}
Since this holds for all $i$, we find that for all $i$, $j$, $(\gamma + 2\truths_i - 2\boosts_i) \boosts_i= (\gamma + 2\truths_j - 2\boosts_j) \boosts_j=\lambda$. We partition $\{1, ..., n\}$ into sets $I$ and $M$, where $I$ contains all $i$ such that $\boosts_i < 1$, and $M$ those where $\boosts_i=1$. For $i\in I$, by noting that using the complementary slackness condition $\gamma_i=0$, this induces a quadratic equation in $\boosts_i$ with solutions
\begin{equation}
    \label{eq:boostv_p2}
    \boosts_i = \frac{1}{2}(\truths_i \pm \sqrt{\truths_i^2-2\lambda}).
\end{equation} 
Since we assume $\boosts_i\geq \truths_i$, we have to take the solution that adds the root of the determinant, that is, $\boosts_i=\frac{1}{2}(\sqrt{\truths_i^2-2\lambda}+\truths_i)$. Furthermore, since we constrain for $i\in I$ that $\boosts_i < 1$, we find that 
\begin{align*}
    1 &> \frac{1}{2}(\truths_i + \sqrt{\truths_i^2-2\lambda}) \\
    2 - \truths_i & > \sqrt{\truths_i^2 - 2\lambda} \\
    \lambda &> 2\truths_i - 2.
\end{align*}
Therefore, given some chosen value of $c$, we require for all $i\in I$ that $\lambda > 2\truths_i - 2$, and so, 
\begin{equation*}
    \min_{i\in I} 2\truths_i - 2 >\lambda 
\end{equation*} 

Unfortunately, finding the exact value of $\lambda$ such that $T_P(\boostv)=\revis{T_P}$ is a challenge. Filling in $\boosts_i$, we find
\begin{equation}
    T_P(\boostv)=\prod_{i=1}^n (\boosts_i)=\prod_{i\in I}\frac{1}{2}(\truths_i+\sqrt{\truths_i^2-2\lambda})=\revis{T_P}.
\end{equation}
This is a $2n$-th degree polynomial in $\lambda$, and we were not able to find an obvious, general closed form solution to it. Mathematica \cite{inc.MathematicaVersion12} finds a complicated closed form formula for $n=2$, but cannot find closed form formulas for $n>2$. 

We also still need to figure out how to partition $i=\{1, ..., n\}$ into $I$ and $M$. Since $\boosts_i$ as computed by Equation \ref{eq:boostv_p2} is a strictly decreasing function in $\lambda$ for all $i\in I$, we have the following unproven proposition. It supports the result given in Theorem \ref{theorem:schur-concave-t-norm} .
\begin{proposition}
    For all $\lambda\in[\min_{i=1}^n 2\truths_i-2, 0]$, the function 
    \begin{equation}
        \minboost_{T_P}(\truth, \lambda)_i= \begin{cases}
            \frac{1}{2}(\truths_i + \sqrt{\truths_i^2 - 2\lambda}) & \text{if } 2\truths_i -2 > \lambda, \\
            1-\truths_i & \text {otherwise.}
        \end{cases}
    \end{equation}
    has the following properties:
    \begin{enumerate}
        \item $\minboost_{T_P}(\truth, \lambda)$ is a minimal \boost\ vector for the product t-norm, the L2 norm and $\revis{T_P}=T_P(\minboost_{T_P}(\truth, \lambda))$;
        \item $\revis{T_P}=T_P(\minboost_{T_P}(\truth, \lambda))$ is a strictly decreasing function in $c$ on $(\min_{i=1}^n 2\truths_i-2, 0]$, and so there is a bijection between $\lambda$ and $\weight \in [T_P(\truth), 1]$ on this interval.
    \end{enumerate}
\end{proposition}
% Do we have a proof for 1 yet?

The second property is easy to see by noting the derivative of $\minboost_{T_P}(\truth, \lambda)$ is negative on $\lambda \in (\min_{i=1}^n 2\truths_i - 2]$, but for the first we do not have a direct proof as of yet and leave this for future work.

Although $\minboost_{T_P}(\truth, \lambda)$ is not parameterized in terms of $\revis{T_P}$, it can still be used in practical scenarios where $\lambda$ can  be seen as the negative \say{confidence} in the clause. A practical implementation could learn a weight for the clause between 0 and 1, and then transform it to the domain of $\lambda$ by dividing by $\min_{i=1}^n 2\truths_i - 2$. Alternatively, $\|T_P(\truth+\minboost_{T_P}(\truth, \lambda))-\revis{T_P}\|_2$ can be minimized with respect to $\lambda$ using mathematical optimization methods like gradient descent or Newton's method to find answers in terms of $\revis{T_P}$.

\section{Additional experiments}
\label{appendix:additional-experiments}
In this Appendix we present additional experiments when $\revis{\varphi}$ is not 1. 
\subsection{Results - \Boosted\ value 0.3}
The figures in this section present the results when the \boosted\ value $\revis{\varphi}=0.3$.
\begin{figure}[H]
    \includegraphics[width=\linewidth]{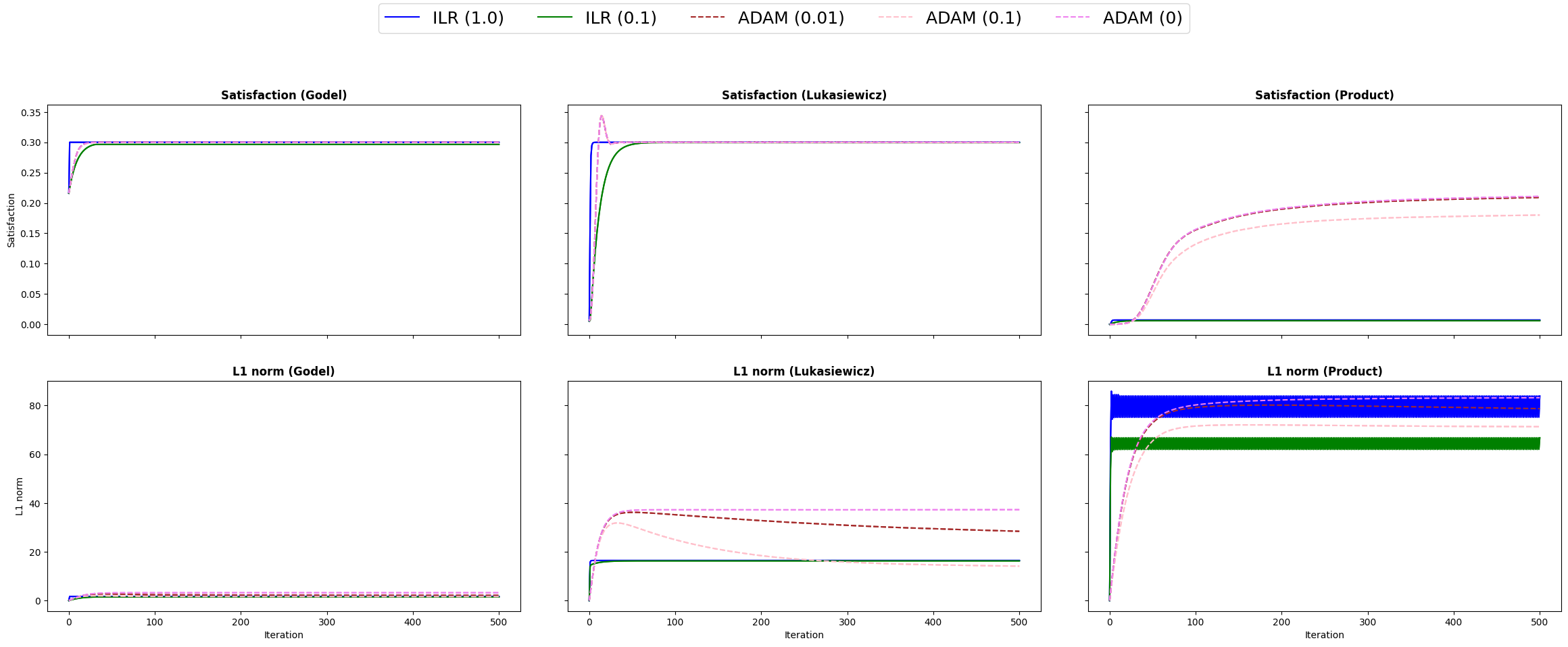}
    \caption{\ale{Comparison of ILR with ADAM on uf20-91 of SATLIB. \Boosted\ value 0.3. The x axis corresponds to the number of iterations, while the y axis is the value of $\revis{\varphi}$ in the first row of the grid, and the L1 norm in the second row.}}
    \label{fig:results_91_0.3}
\end{figure}

\begin{figure}[H]
    \includegraphics[width=\linewidth]{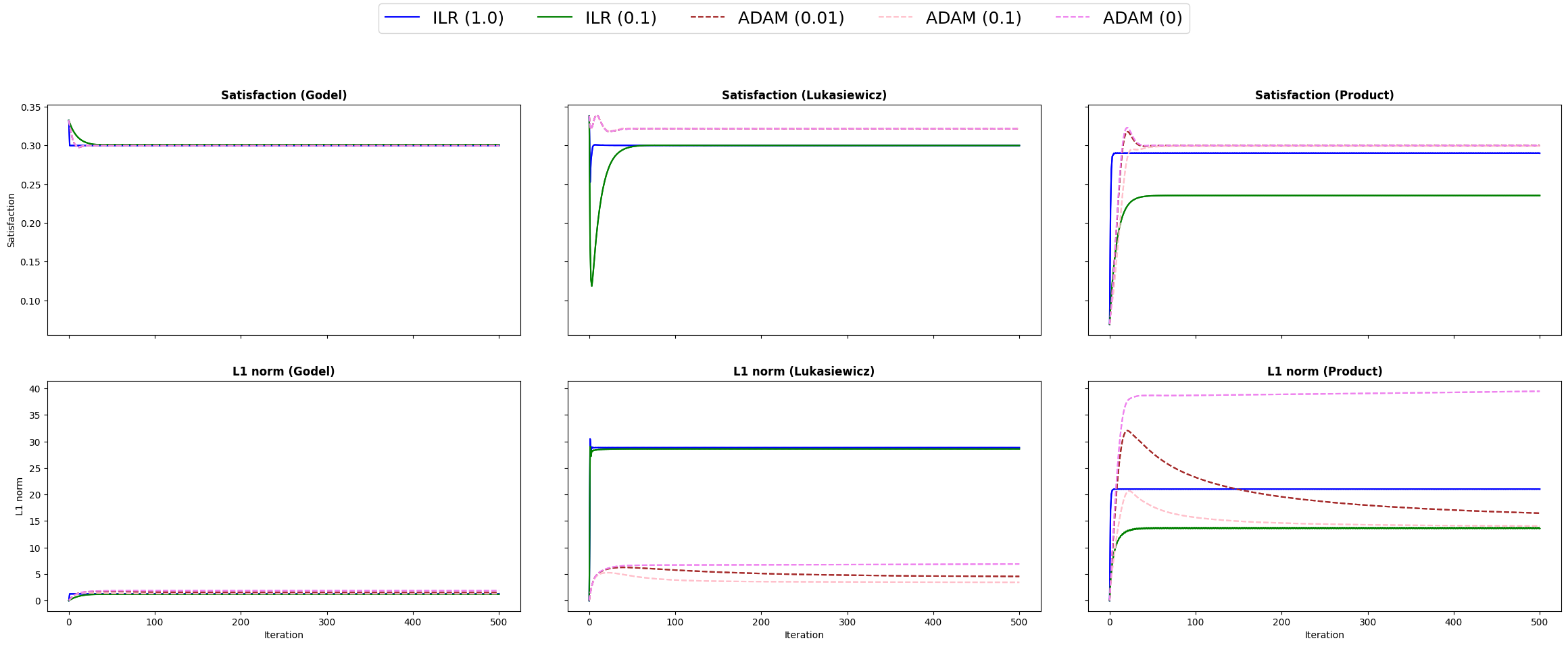}
    \caption{\ale{Comparison of ILR with ADAM on the uf20-91 with 20 clauses. \Boosted\ value 0.3.}}
    \label{fig:results_91_0.3}
\end{figure}

\subsection{Results - \Boosted\ value 0.5}
The figures in this section present the results when the \boosted\ value $\revis{\varphi}=0.5$. We note that the satisfaction for ADAM in \luk\ converges above 0.5 in Figure \ref{fig:results_91_0.5}. This means the final truth value is too high, and it has not found a proper solution here. 

\begin{figure}[H]
    \includegraphics[width=\linewidth]{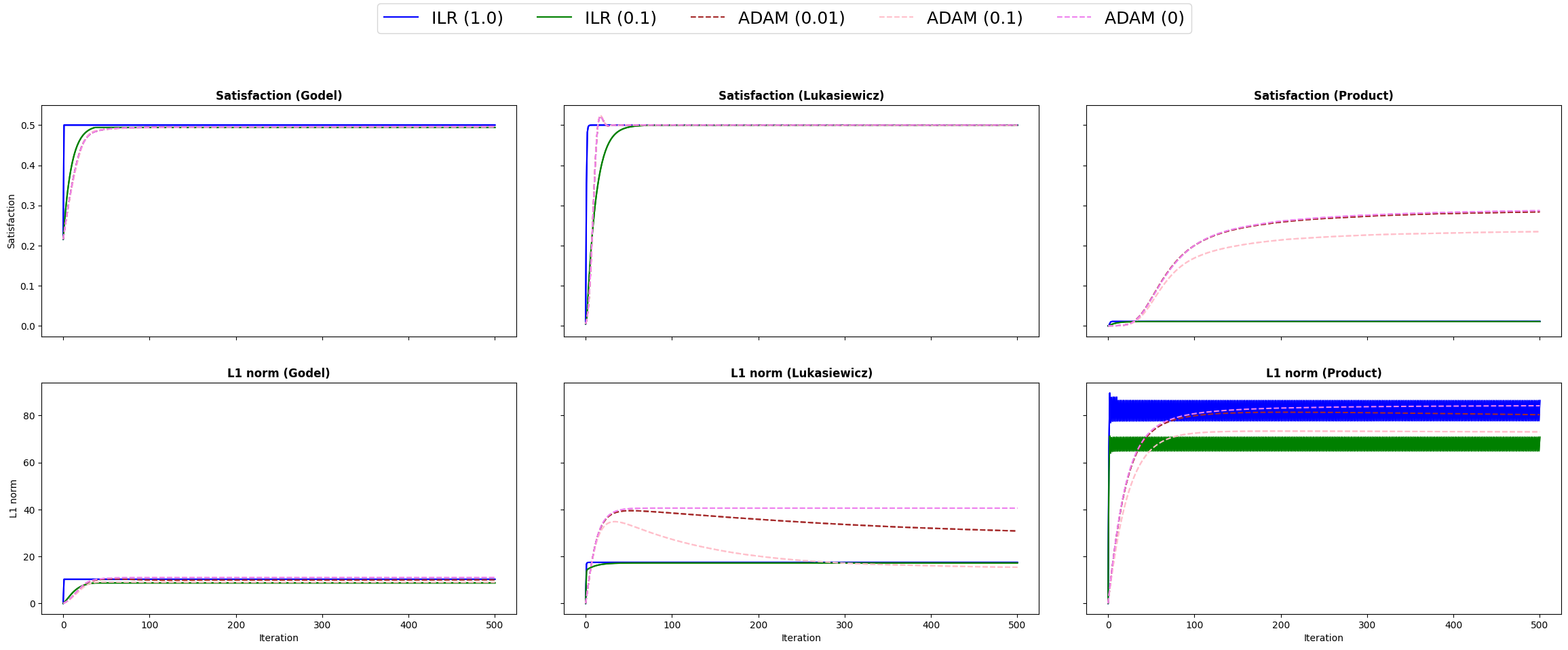}
    \caption{\ale{Comparison of ILR with ADAM on the uf20-91 of SATLIB. \Boosted\ value 0.5.}}
    \label{fig:results_91_0.5}
\end{figure}

\begin{figure}[H]
    \includegraphics[width=\linewidth]{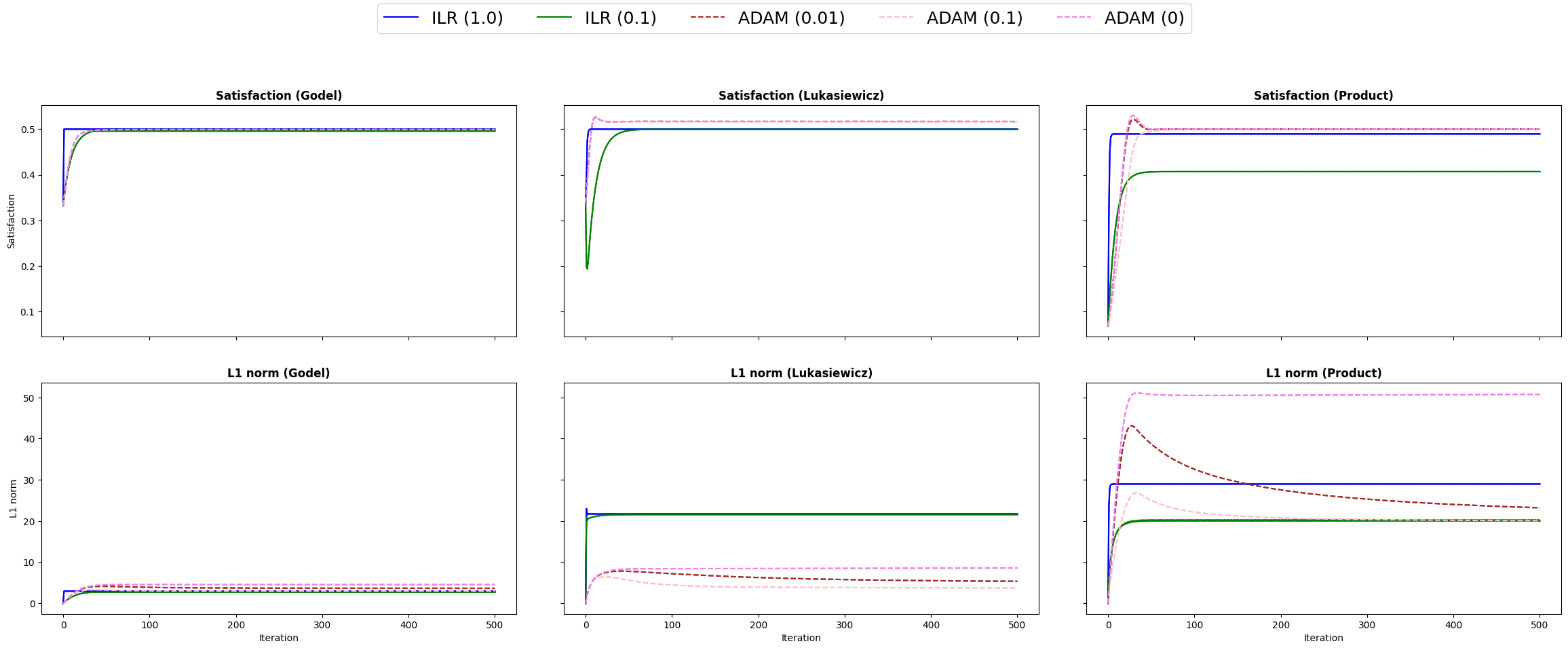}
    \caption{\ale{Comparison of ILR with ADAM on the uf20-91 with 20 clauses. \Boosted\ value 0.5.}}
    \label{fig:results_91_0.5}
\end{figure}

\end{document}